\documentclass[twoside]{article}
\pdfoutput=1

\usepackage[accepted]{aistats2021_mod}
\usepackage[utf8]{inputenc}
\usepackage{amssymb}
\usepackage[ruled,vlined]{algorithm2e}
\usepackage{amsthm}
\usepackage{amsmath}
\usepackage{thmtools}
\usepackage{thm-restate}
\usepackage{graphicx}
\usepackage{subcaption}
\usepackage{color-edits}
\usepackage{comment}
\usepackage{soul}

\newtheorem{theorem}{Theorem}
\newtheorem{definition}{Definition}
\newtheorem{lemma}{Lemma}
\newtheorem{fact}{Fact}

\usepackage[round]{natbib}

\newcommand{\bmu}{\pmb{\mu}}
\newcommand{\bc}{\pmb{c}}
\newcommand{\prob}{\Pr}

\newcommand{\hmu}{\hat{\mu}}
\newcommand{\hc}{\hat{c}}
\newcommand{\tdt}{\tilde{t}}
\newcommand{\ex}{\mathbf{E}}
\newcommand{\ind}{\mathbf{I}}
\newcommand{\alg}{\pi_{\alpha}}
\newcommand{\algh}{\pi_0 }
\newcommand{\bfac}{\mathit{BernoulliFactory}}
\newcommand{\cI}{\Phi_{\theta, p, \epsilon}}
\newcommand{\cIzero}{\Phi_{0, p, \epsilon}}
\newcommand{\cIalpha}{\Phi_{\alpha, p, \epsilon}}

\newcommand{\rIzero}{\phi_{0, p, \epsilon}}

\newcommand{\qreg}{{\sf Quality\_Reg}}
\newcommand{\creg}{{\sf Cost\_Reg}}

\newcommand{\mqreg}{{\sf Mod\_Quality\_Reg}}
\newcommand{\eg}{\emph{e.g., }}
\newcommand{\ie}{\emph{i.e., }}
\newcommand{\MAB}{\textsf{MAB }}
\newcommand{\MABnosapce}{\textsf{MAB}}
\newcommand{\muucb}{\mu^{\sf UCB}}
\newcommand{\mulcb}{\mu^{\sf LCB}}
\newcommand{\ber}{Ber}
\newcommand{\optarm}{i_*}
\newcommand{\mainALG}{CS-ETC}%{\textsf{CS-ETC}}
\newcommand{\consTS}{CS-TS}%{\textsf{CS-TS}}
\newcommand{\consUCB}{CS-UCB} %{\textsf{CS-UCB}}
\newcommand{\consTSfull}{\textsf{Cost-Subsidized TS}}
\newcommand{\consUCBfull}{\textsf{Cost-Subsidized UCB}}
\newcommand{\mainALGfull}{\textsf{Cost-Subsidized Explore-Then-Commit}}

\addauthor{ka}{cyan}   % ka for Karthik
\addauthor{ds}{green}

\graphicspath{{plots/}}

\begin{document}

%\maketitle

\twocolumn[

\aistatstitle{Multi-Armed Bandits with Cost Subsidy}

\aistatsauthor{ Deeksha Sinha \And Karthik Abinav Sankararaman \And  Abbas Kazerouni \And Vashist Avadhanula}

\aistatsaddress{ MIT/Facebook \And  Facebook \And  Facebook \And  Facebook } ]
%\aistatsaddress{ deeksha@mit.edu \And  Facebook \And  kazerouni@fb.com \And  Facebook } ]

\begin{abstract}
In this paper, we consider a novel variant of the multi-armed bandit (MAB) problem, \textit{\MABnosapce\; with cost subsidy}, which models many real-life applications where the learning agent has to pay to select an arm and is concerned about optimizing cumulative costs and rewards. We present two applications, \emph{intelligent SMS routing problem} and \emph{ad audience optimization problem} faced by several businesses (especially online platforms), and show how our problem uniquely captures key features of these applications. We show that naive generalizations of existing \MABnosapce\; algorithms like Upper Confidence Bound and Thompson Sampling do not perform well for this problem. We then establish a fundamental lower bound  %of $\Omega(K^{1/3} T^{2/3})$ 
on the performance of any online learning algorithm for this problem, highlighting the hardness of our problem in comparison to the classical \MABnosapce\; problem. %(where $T$ is the time horizon and $K$ is the number of arms). 
We also present a simple variant of \textit{explore-then-commit} and establish near-optimal regret bounds for this algorithm. Lastly, we perform extensive numerical simulations to understand the behavior of a suite of algorithms for various instances and recommend a practical guide to employ different algorithms. 
\end{abstract}

\section{Introduction}
\label{sec:intro}
%Multi-armed bandit (\MABnosapce) problem is a popular statistical framework that elegantly balances the exploration-exploitation trade-offs associated with online learning problems. 
In the traditional (stochastic) \MABnosapce\; problem (\cite{robbins1952some}), the learning agent has access to a set of $K$ actions (arms) with unknown but fixed reward distributions and has to repeatedly select an arm to maximize the cumulative reward. Here, the challenge is designing a policy that balances the tension between acquiring information about arms with little historical observations and exploiting the most rewarding arm based on existing information. The aforementioned exploration-exploitation trade-off has been extensively studied, leading to some simple but extremely effective algorithms like Upper Confidence Bound (\cite{Auer2002}) and Thompson Sampling (\cite{thompson, AgrawalTS_nearopt}), which have been further generalized and applied in a wide range of application domains, including online advertising (\cite{langford2008epoch, cesa2014regret}, \cite{chapelle}), recommendation systems (\cite{li2015counterfactual, li2011unbiased, agrawalnear}), social networks and crowd sourcing (\cite{anandkumar2011distributed,sankararaman2019social}, \cite{slivkins2014online}); see \cite{MAL-024} and \cite{MAL-068} for a detailed review. However, most of these approaches cannot be generalized to settings involving multiple metrics (for example, reward and cost) when the underlying trade-offs between these metrics are not known a priori. 

In many real-world applications of \MABnosapce, some of which we will elaborate below, it is common for the agent to incur costs to play an arm, with \textit{high performing arms} costing more. Though one can model this in the traditional \MAB framework by considering cost subtracted from the reward as the modified objective, such a modification is not always meaningful, particularly in settings where the reward and cost associated with an arm represent different quantities (for example, click rate and cost of an ad). In such problems, it is natural for the learning agent to \textit{optimize} for both the metrics, typically avoiding incurring exorbitant costs for a marginal increase in cumulative reward. Motivated by the aforementioned scenario, in this paper, we consider a variant of the \MAB problem, where the agent is not only concerned about balancing the exploration-exploitation trade-offs to maximize the cumulative reward but also balance the trade-offs associated with multiple objectives that are intrinsic to several practical applications. More specifically, in this work, we study a stylized problem, where to manage costs, the agent is willing to tolerate a small loss from the \textit{highest reward}, measured as the reward that the traditional \MAB problem could obtain in the absence of costs. We refer to this problem as \MAB problem with a cost subsidy (see Section  \ref{sec:prob_form} for exact problem formulation), where the subsidy refers to the amount of reward the learning agent is willing to forgo to improve costs. Before explaining our problem and technical contributions in detail, we will elaborate on the applications that motivate this problem. 

\medskip \noindent {\bf Intelligent SMS Routing.} Many businesses such as banks, delivery services, airlines, hotels, and various online platforms send SMSes (text messages) to their users for various reasons, including two-factor authentication, order confirmations, appointment reminders, transaction alerts, and as a direct marketing line (see \cite{twiliober}). These text messages, referred to as Application-to-Person (A2P) messages, constitute a significant portion of all text messages sent through cellular networks today. In fact, A2P messages are forecasted to be a \$86.3 billion business by 2025 (\cite{MarketWatch}). 

To deliver these messages, businesses typically enlist the support of telecom aggregators, who have private agreements with mobile operators. Each aggregator offers a unique combination of quality, as measured by the fraction of text messages successfully delivered by them and price per message. Surprisingly, it is common for delivery rates of text messages not to be very high (see \cite{canlas2010quantitative,meng2007analysis,zerfos2006study,osunaderoute} for QoS analysis in different geographies) and for aggregator's quality to fluctuate with time due to various reasons ranging from network outage to traffic congestion. Therefore, the platform's problem of balancing the tension between inferring aggregator's quality through exploration and exploiting the current \textit{best performing aggregator} to maximize the number of messages delivered to users leads to a standard \MAB formulation. However, given the large volume of messages that need to be dispatched, an \MAB based solution that focuses exclusively on the quality of the aggregator could result in exorbitant spending for the business. A survey of businesses shows that the number of text messages they are willing to send will significantly drop if the cost per SMS is increased by a few cents per SMS (\cite{Ovum}). Moreover, in many situations, platforms have backup communication channels such as email-based authentication or notifications via in-app/website features. Though not as effective as a text message in terms of read rate, it can be used if guaranteeing the text message delivery proves to be very costly. %in case the SMS did not reach the user (for e.g. email based authentication or notifications via in-app/website features). 
Therefore, it is natural for businesses to prefer an aggregator with lower costs as long as their quality is comparable to the aggregator with the best quality. 

\medskip \noindent {\bf Ad-audience Optimization.} We now describe another real-world application in the context of online advertisements. Many advertisers (especially small-to-medium scale businesses) have increasingly embraced the notion of \emph{auto-targeting} where they let the advertising platform identify a \emph{high-quality} audience group (\eg \cite{koningstein2006suggesting, Amazon, Facebook, Google}). To enable this, the platform explores the audience-space to identify \emph{cheaper} opportunities that also provide high click-through-rate (ctr) and conversion rate. Here, different audience groups can have different yields, i.e., quality (CTR/conversion rate) for a specific ad. However, it may require vastly different bids to reach different audiences due to auction overlap with other ad campaigns with smaller audience targeting. Thus, the algorithm is faced with a similar trade-off; as long as a particular audience-group gives a high-yield, the goal is to find the cheapest one.

We now present a novel formulation of a multi-armed bandit problem that captures the key features of these applications. Our goal is to develop a cost-sensitive \MAB  algorithm that balances both the exploration-exploitation trade-offs as well as the tension between conflicting metrics in a multi-objective setting. 

%associated with learning news well suited for this (and other similar) application(s). We study this model both theoretically and empirically in this paper.

\subsection{Problem Formulation}\label{sec:prob_form}
    To formally state our problem, given an instance $I$, in every round $t \in [T]$ the agent chooses an arm $i \in [K]$ and realizes a reward $r_{t},$ sampled independently from a fixed, but unknown distribution $\mathcal{F}_i$ with mean $\mu_i$ (or $\mu_i^I$) and incurs a cost $c_{i}$ (or $c_i^I$), which is known a priori. Here, to manage costs, we allow the agent to be agnostic between arms, whose expected reward is greater than $1- \alpha$ fraction of the highest expected reward, for a fixed and known value of $\alpha$, which we refer to as the \textit{subsidy factor}. The agent's objective is to learn and pull the cheapest arm among these \textit{high}-quality arms as frequently as possible. 
    
   More specifically, let $m_{*}$ denote the arm with highest expected mean, \ie $m_{*} = {\text{argmax}_{i \in [K]}} \; \mu_i,$  and $\mathcal{C}_{*}$ be the set of arms whose expected reward is within $1- \alpha$ factor of the highest expected reward, \ie $\mathcal{C}_{{*}} = \{i \in [K] \;|\; \mu_i \geq (1-\alpha) \mu_{m_{*}}\}.$ We refer to the quantity $(1-\alpha)\mu_{m_{*}}$ as the \textit{smallest tolerated reward}. Without loss of generality, we assume the reward distribution has support $[0,1]$. The agent's goal is to design a policy (algorithm) $\pi$ that will learn the cheapest arm whose expected reward is at least as large as the smallest tolerated reward. In other words, the agent needs to learn the identity and simultaneously maximize the number of plays of arm $\optarm={\text{argmin}_{i \in  \mathcal{C}_{*}}}\; c_i.$  Since in the SMS application, the reward is the quality of the chosen aggregator, we will use the terms reward and quality interchangeably.  

To measure the performance of any policy $\pi$, we propose two notions of regret - quality and cost regret, with the agent's goal being minimizing both of them:
\begin{equation}\label{eq:regret}
\begin{aligned}
& \qreg_\pi(T, \alpha,  \bmu, \bc) \\ &\qquad = \ex \left[ \sum_{t=1}^T \max\{ (1-\alpha) \mu_{m_{*}} - \mu_{\pi_t},0\} \right], \\
& \creg_\pi(T, \alpha, \bmu, \bc) \\ & \qquad = \ex \left[ \sum_{t=1}^T \max\{c_{\pi_{t}} - c_{\optarm},0\} \right],
\end{aligned}
\end{equation}
where $\bc=(c_1, \cdots c_K), \bmu = (\mu_1, \cdots \mu_K)$ and the expectation is over the randomness in policy $\pi$. Equivalently, the cost and quality regret of policy $\pi$ on an instance $I$ of the problem is denoted as $\qreg_\pi(T, \alpha,I)$ and $\creg_\pi(T, \alpha,I)$ where the instance is defined by the distributions of the reward and cost of each arm. %On some occasions, we also use the notation of $\bc^I, \bmu^I$ to refer to the costs and mean rewards of instance $I$. 
The objective then is to design a policy that simultaneously minimizes both the cost and quality regret for all possible choices of $\bmu$ and $\bc$ (equivalently all instances I). %When $I$, $\pmb{\mu}$ and $\mathbf{c}$ are clear from the context, we may drop them from the above notations. %Here, the agent's goal is to design an algorithm that minimizes both the cost and quality regret for any choice of $\pmb{\mu}$ and $\mathbf{c}.$

\noindent {\bf Choice of Objective Function.}
    %\begin{remark} [Choice of objective function]
    % \label{rem:LagFunction}
     Note that a parametrized linear combination of reward and cost metrics, i.e., $\mu-\lambda c$ for an appropriately chosen $\lambda$ is a popular approach to balance cost-reward trade-off. From an application stand-point, there are two important considerations that favor using our objective instead of a linear combination of reward and cost metrics. First, in a real-world system, we need explicit control over the parameter $\alpha$ that is not instance-dependent to understand and defend the trade-off between the various objectives. From a product manager's perspective, the \textit{indifference} among the arms within $(1-\alpha)$ of the arm with the highest mean reward is a transparent and interpretable compromise to lower costs. Second, for the intelligent SMS routing application discussed earlier, different sets of aggregators operate in different regions. Thus, separate $\lambda$ values would need to be configured for each region, making the process cumbersome. %even if they optimize for the same set of constraints with the same set of instance means.
    %\end{remark} 
     
    %\begin{remark} [Inequivalence with a linear combination based objective function]
    % \label{rem:inequivalence}
     Further note that, the setting considered in this paper is not equivalent to taking a parametrized linear combination of reward and cost metrics. In particular, for any specified subsidy factor $\alpha$, the value $\lambda$  required in the linear objective function, for $\optarm$ to be the optimal arm would depend on the cost and reward distributions of the arms. Therefore, using a single value of $\lambda$ and relying on standard \MAB algorithms would not lead to the desired outcome for our problem. 
     Moreover, in our setting,  the cost and quality regret measures are defined with respect to different benchmark arms ($i_*$ and $m_*$ respectively). The reward of the played arm is compared to $(1-\alpha)$ times the reward of the highest mean reward arm and not directly the mean reward of the arm $m_*$ as would be the case in a single objective based formulation.  

\subsection{Related Work}
Our problem is closely related to the $\MABnosapce$ with multiple objectives line of work, which has attracted considerable attention in recent times. The existing literature on multi-objective $\MABnosapce$ can be broadly classified into the following categories. 

\noindent {\bf Bandits with Knapsacks (BwK).} Bandits with knapsacks (BwK), introduced in the seminal work of  \cite{Badanidiyuru:2018} is a general framework that considers the standard \MABnosapce\;problem under the presence of additional budget/resource constraints. The BwK problem encapsulates a large number of constrained bandit problems that naturally arise in many application domains, including dynamic pricing, auction bidding, routing, and scheduling (see \cite{tran2012knapsack,AD14,Immorlica19}). In this formulation, the agent has access to a set of $d$ finite resources and $K$ arms,  each associated with a reward distribution. Upon playing arm $a$ at time $t$, the agent realizes a reward of $r_t$ and incurs a penalty of $c^{(i)}_t$ for resource $i$, all drawn from a fixed, but unknown distribution corresponding to the arm. The objective of the agent is to maximize the cumulative reward before one of the resources is completely depleted. Although appealing in many applications, BwK formulation requires \emph{hard} constraint on resources (cost in our setting) and hence, cannot be easily generalized to our problem. In particular, in the cost subsidized \MAB problem, the equivalent budget limits depend on the problem instance and therefore cannot be determined a priori. 

\noindent {\bf Pareto Optimality and Composite Objective.}  The second formulation focuses on identifying Pareto optimal alternatives and uniformly choosing among these options  (see \cite{drugan2013designing, yahyaa2014annealing, paria2018flexible, yahyaa2015thompson}). These approaches do not apply to our problem. Some of the Pareto alternatives could have extreme values for one of the metrics, for example, having a meager cost and low quality or extremely high cost and quality, making them undesirable for the applications discussed earlier. Closely related to this line of work is the set of works that focus on a composite objective by appropriately weighting the different metrics (see \cite{paria2018flexible, yahyaa2015thompson}). Such formulations also do not immediately apply to our problem. In the SMS and ad applications discussed earlier, it is not acceptable to drop the quality beyond the allowed level irrespective of the cost savings we could obtain. Furthermore, in the SMS application, the trade-offs between quality and costs could vary from region to region, making it hard to identify a good set of weights for the composite objective (see Section \ref{sec:prob_form}).

\noindent {\bf Conservative Bandits and Bandits with Safety Constraints.} Two other lines of work that are recently receiving increased attention, particularly from practitioners, are \textit{bandits with safety constraints} (see \cite{daulton2019thompson,amani2020generalized,galichet2013exploration}) and \textit{conservative bandits} (see \cite{wu2016conservative, kazerouni2017conservative}). In both these formulation, the algorithm chooses one of the arms and receives a reward and a cost associated with it.
The goal of the algorithms is to maximize the total reward obtained while ensuring that either the chosen arm is within a pre-specified threshold (when costs of arms are unknown a priori) or the reward of the arm is at least a specified fraction of a known benchmark arm. Neither of these models exactly captures the requirements of our applications: a) we do not have a hard constraint on the acceptable cost of a pulled arm. In particular, choosing low-quality aggregators to avoid high costs (even for a few rounds) could be disastrous since it leads to bad user experience on the platform and eventual churn, and b) the equivalent benchmark arm in our case, i.e., the arm with the highest mean reward is not known a priori.

\noindent {\bf Best Arm Identification.} Apart from the closely related works mentioned above, our problem of identifying the \textit{cheapest} arm whose expected reward is within an acceptable margin from the highest reward can be formulated as a stylized version of the \emph{best-arm identification problem} (\cite{katz2019top,jamieson2014best,chen2014combinatorial,cao2015top,chen2016pure}). However, in many settings and particularly applications discussed earlier, the agent's objective is optimizing cumulative reward and not just identifying the \textit{best arm}. %Therefore, the exclusive focus on exploration to ensure \textit{high confidence} of identifying the \textit{best arm} associated with this line of work would make it less appealing for our problem.  

    % This can be modelled using our framework as follows. Denote each of the large audience category (usually segmented by top interests) as one of the $K$ arms. Each arm is associated with an unknown click-through-rate and a cost per action (such as magnitude of bid). Here the time-step is some segmentation of the overall length of the campaign (\eg every minute). At each time-step, the algorithm chooses an arm $i \in [K]$ and matches this ad-campaign to requests originating from the corresponding audience group. The algorithm receives feedback in the form of clicks and average bid required to win the auctions. The goal of the algorithm is to obtain ctr that is comparable to the best performing audience segment while minimizing the total cost. 
 
 \subsection{Our Contributions}
 \noindent {\bf Novel Problem Formulation.} In this work, we propose a stylized model, \textit{\MAB with a cost subsidy}, and introduce new performance metrics that uniquely capture the salient features of many real-world online learning problems involving multiple objectives. For this problem, we first show that naive generalization of popular algorithms like Upper Confidence Bound (UCB) and Thompson Sampling (TS) could lead to poor performance on the metrics. In particular, we show that the naive generalization of TS for this problem would lead to a linear cost regret for some problem instances. 
 
 \noindent {\bf Lower Bound.} We establish a fundamental limit on the performance of \emph{any online algorithm} for our problem. More specifically, we show that any online learning algorithm will incur a regret of $\Omega(K^{1/3}T^{2/3})$ on either the cost or the quality metric (refer to (\ref{eq:regret})), further establishing the \textit{hardness of our problem} relative to the standard \MABnosapce\; problem, for which it is possible to design algorithms that achieve worst-case regret bound of $\tilde{O}(\sqrt{KT})$. We introduce a novel reduction technique to derive the above lower bound, which is of independent interest.  
 
 \noindent {\bf Cost Subsidized Explore-Then-Commit.} We present a simple algorithm based on the \textit{explore-then-commit} (ETC) principle and show that it achieves near-optimal performance guarantees. In particular, we establish that our algorithm achieves a worst-case bound of $O(K^{1/3}T^{2/3}\sqrt{\log{T}})$ for both cost and quality regret. A key challenge in generalizing the ETC algorithm for this problem arises from having to balance between two asymmetric objectives. %To this end, we suitably modify the standard ETC idea to account for this asymmetry and ensure both the regret bounds are minimized. 
 We also discuss generalizations of the algorithm for settings where the cost of the arms is not known a priori. Furthermore, we consider a special scenario of bounded costs, where naive generalizations of TS and UCB work reasonably well and establish worst-case regret bounds.

 \noindent {\bf Numerical Simulation.} Lastly, we perform extensive simulations to understand various regimes of the problem parameters and compare different algorithms. More specifically, we consider scenarios where naive generalizations of UCB and TS, which have been adapted in real-life implementations (see \cite{daulton2019thompson}) perform well and settings where they perform poorly, which should be of interest to practitioners.  
 
 %which we refer to as cost subsidized explore first 
 
 %Second, we show that prior approaches to related discrete \MAB models do not give sub-linear regret for this problem setting. Thus, a modification to existing bandit models will not transfer to the applications stated before. Third, we show that this problem is harder than the typical bandit problems on discrete set of arms; we show that any algorithm incurs a regret of at least $\Omega(T^{2/3})$ even when one of the objectives is deterministic. Fourth, we show that a small modification of the explore-first algorithm can indeed match this lower-bound. Finally, we perform extensive simulations to understand the various regimes in the problem parameters and compare different algorithms. We perform an exhaustive empirical study that we believe will be enlightening for applications to multiple domains.

    \subsection{Outline}
    The rest of this paper is structured as follows. In Section~\ref{sec:ts_linear_regret}, we show that the naive generalization to TS or UCB algorithms performs poorly, and in Section ~\ref{sec:lowerBounds}, we establish lower bounds on the performance of any algorithm for \MABnosapce\;with cost subsidy problem. In Section~\ref{sec:exploreFirst}, we present a variation of the ETC algorithm and show that it achieves a near-optimal regret bound of $\tilde{O}(K^{1/3}T^{2/3})$ for both the metrics. In section~\ref{sec:specialCases}, we show that with additional assumptions, it is possible to show improved performance bounds for naive generalization of existing algorithms. Finally, in section~\ref{sec:experiments}, we perform numerical simulations to explore various regimes of the instance-space.
    
\section{Performance of Existing MAB Algorithms}
\label{sec:ts_linear_regret}
In this section, we consider a natural extension of two popular \MAB  algorithms, TS and UCB, for our problem and show that such adaptations perform poorly. This highlights the challenges involved in developing good algorithms for the \MAB problem with cost subsidy. In particular, we establish theoretically that for some problem instances, the TS variant incurs a linear cost regret and observe similar performance for the UCB variant empirically. Our key focus on TS in this section is primarily motivated by the superior performance observed over a stream of recent papers in the context of TS versus more traditional approaches such as UCB (see \cite{scott2010modern,chapelle,may2012optimistic,agrawal2017thompson}).

We present the details of TS and UCB adaptations in Algorithm \ref{alg::ts}, which we will refer to as \consTSfull (\consTS)\; and \consUCBfull(\consUCB)\;respectively. These extensions are inspired by \cite{daulton2019thompson}, which demonstrates empirical efficacy on a related (but different) problem. Briefly, in the \consTS (\consUCB)\;variation, we follow the standard TS (UCB) algorithm and obtain a quality score which is a sample from the posterior distribution (upper confidence bound) for each arm. We then construct a feasible set of arms whose quality scores are greater than   $1-\alpha$  fraction of the highest quality score. Finally, we pull the cheapest arm among the feasible set of arms.

%At each period in this algorithm, TS or UCB offers an estimate for the quality of each arm, based on the observed qualities so far. Next, a feasible set is constructed with respect to the best estimated quality, and then the feasible arm with lowest cost is played. For \consTS, a prior $\nu_i$ on the reward distribution of each arm $i$ is needed. $\nu_i\big(\cdot|\{r_s\}_{s \in \{1, 2, \cdots t-1 \} \text{ s.t. } I_s = i}\big)$ denotes the posterior of the reward distribution of arm $i$ after observing the rewards obtained on previous pulls of arm $i$.

We will now show that \consTS\; with Gaussian priors and posteriors (i.e., Gaussian distribution with mean $\hat{\mu}_i(t)$ and variance $1/T_i(t)$) described in Algorithm \ref{alg::ts} incurs a linear cost regret in the worst case. More precisely, we prove the following result.

% \begin{theorem}\label{thm::ts-lin-reg}
% When the reward distribution of all arms and the prior distribution of mean rewards are assumed to be Gaussian, for any given $\alpha, K, T$  there exists an instance $\phi$ of problem (\ref{eq:regret}) such that
    % $\qreg_{{\sf \consTS }}(T, \alpha, \phi) + \creg_{{\sf \consTS }}(T,\alpha,\phi)$ = $\Omega\left(T\right)$.
\begin{theorem}\label{thm::ts-lin-reg}
For any given $K, \alpha, T$  there exists an instance $\phi$ of problem such that
    $\qreg_{{\sf \consTS }}(T,\alpha,\phi) + \creg_{{\sf \consTS }}(T,\alpha,\phi)$  is $\Omega\left(T\right)$.
\end{theorem}    
% \end{theorem}
\begin{proof}[Proof Sketch]
The proof closely follows the lower bound argument in \cite{agrawal2017near}. 
We briefly describe the intuition behind the result.  Consider a scenario where the highest reward arm is an expensive arm while all other arms are cheap and have rewards marginally above the smallest tolerated reward.  In the traditional \MABnosapce\;problem,  the anti-concentration property of the Gaussian distribution (see \cite{agrawal2017near}) ensures samples from the \textit{good arm} would be large enough with sufficient frequency, ensuring appropriate exploration and good performance. However, in our problem, the anti-concentration property would result in playing the expensive arm too often since the difference in the mean qualities is small,  incurring a linear cost regret while achieving zero quality regret. A complete proof of the theorem is provided in Appendix \ref{sec:alg_performance_appendix}.
\end{proof}

The algorithm's poor performance is not limited only to the above instance and usage of Gaussian prior. More generally, the \consTS \ and \consUCB \ algorithms seem to perform poorly whenever the mean reward of the optimal arm is very close to the smallest tolerated reward. We illustrate this through another empirical example. Consider the following instance with two arms, each having Bernoulli rewards  and $T=10,000$. The costs of the two arms are  $c_1=0$ and $c_2=1$. The expected qualities are $\mu_1 = 0.5(1-\alpha) + 1/\sqrt{T}, \mu_2=0.5$ with $\alpha = 0.1$. The prior of the mean reward of both the arms is a Beta(1,1) distribution.  Here, the quality regret will be zero irrespective of which arm is played. But both \consTS \ and \consUCB \ incur significant cost regret as shown in Figure \ref{fig:lin_reg}. (In the figure, we also plot the performance of the key algorithm we propose in the paper (Algorithm \ref{alg:explore_first}) and note that it has much superior performance compared to \consTS \ and \consUCB.)

\begin{algorithm}
\SetAlgoLined
%\label{alg::ts}
\KwResult{Arm $I_t$ to be pulled in each round $t \in [T]$ }
\SetKwInOut{Input}{Input}\SetKwInOut{Output}{Output}
\Input{$T, K$, prior distribution for mean rewards of all arms $\{\nu_i\}_{i=1}^K$, reward likelihood function $\{L_i\}_{i=1}^K$}
%\Output{$I_t \ \forall t \in [T]$}
$T_i(1)= 0 \ \forall i \in [K]$\;
\For{$t \in [K]$}{
    $I_t = t$\;
    Play arm $I_t$ and observe reward $r_t$\;
    $T_{i}(t+1) = T_i(t) + \mathbf{1}\{ I_t =i \} \ \forall i \in [K]$\;
    
}
\For{$t \in [K+1, T]$}{
    \For{$i \in [K]$}{
    %     Let $\hq_i(t) = \frac{\sum_{t: I_t = i} r_t}{\frac{\sigma_n^2}{\sigma_0^2}+ T_i(t) }$ and $\sigma_i^2(t) = \frac{1}{ \frac{1}{\sigma_0^2} + \frac{T_i(t)}{\sigma_n^2}}$\;
    %$ \hmu_i(t) \leftarrow \frac{ \sum_{\tau=1}^{t-1} r_{\tau} \ind \{ I_{\tau} =i \}} { T_i(t) } \  $\;
    %$\beta_i(t) \leftarrow \sqrt{\frac{ 2 \log{T} }{T_i(t)}} \ $\;
    $ \hmu_i(t) \leftarrow \left( \sum_{\tau=1}^{t-1} r_{\tau} \ind \{ I_{\tau} =i \} \right)/ { T_i(t) }$\;
    $\beta_i(t) \leftarrow \sqrt{ \left(  2 \log{T} \right) /{T_i(t)}}$\;
      \textbf{UCB}: 
       $\mu^{score}_i(t) \leftarrow \min \{ \hmu_i(t) + \beta_i(t), 1 \} $\;
       \textbf{TS}: 
       Sample $\mu^{score}_i(t)$ from  the posterior distribution of arm $i$, $\nu_i\big(\cdot|\{r_s\}_{s \in \{1, 2, \cdots t-1 \} \text{ s.t. } I_s = i}, L_i \big)$\;
    }
    $m_t = \arg \max_{i} \mu^{score}_i(t)$\;
    $Feas(t) = \{i: \mu^{score}_i(t) - (1-\alpha) \mu^{score}_{m_t} \geq 0 \}$\;
    $I_t = \arg \min_{i \in Feas(t)} c_i $\;
    Play arm $I_t$ and observe reward $r_t$\;
    $T_{i}(t+1) = T_i(t) + \mathbf{1}\{ I_t =i \} \ \forall i \in [K]$\;
}
\caption{\textsf{Cost Subsidized TS and UCB Algorithms}}
\label{alg::ts}
\end{algorithm}

% \begin{minipage}[b]{0.49\textwidth}
\begin{figure}
\centering
	\includegraphics[width=0.8\linewidth]{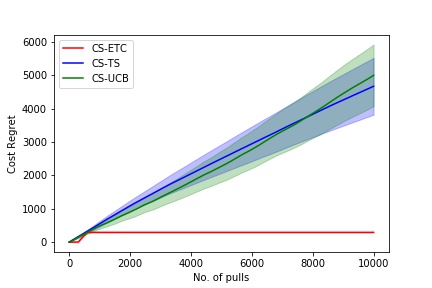}
	\captionof{figure}{Cost regret of various algorithms for an instance where the mean reward of the optimal arm is very close to the smallest tolerated reward. \consTS \ and \consUCB \ incur significant regret. But \mainALG \ attains low cost regret. The width of the error bands is two standard deviations based on 50 runs of the simulation.
	\label{fig:lin_reg}}
\end{figure}
% \end{minipage}

\section{Lower Bound}
\label{sec:lowerBounds}
In this section, we establish that any policy must incur a regret of $\Omega(K^{1/3}T^{2/3})$ on at least one of the regret metrics.
More precisely, we prove the following result.

% \begin{theorem}\label{thm:lower_bound}
% For any given $\alpha, K, T$ and policy $\pi$ that selects arm $\pi_t$ at time $t$,  there exists an instance $\phi$ of problem \eqref{eq:regret} with $K+1$ arms such that $    \qreg_\pi(T, \alpha, \phi) + \creg_\pi(T,\alpha,\phi) $ is  $ \Omega\left((1-\alpha)^2 K^{\frac{1}{3}}T^{\frac{2}{3}}\right)$ when $0 \leq \alpha \leq 1, T \geq 1$ and  $K \geq 1$.
% \end{theorem}

\begin{theorem}\label{thm:lower_bound}
For any given $\alpha, K, T$ and (possibly randomized) policy $\pi$, there exists an instance $\phi$ of problem \eqref{eq:regret} with $K+1$ arms such that $    \qreg_\pi(T, \alpha, \phi) + \creg_\pi(T,\alpha,\phi) $ is  $ \Omega\left((1-\alpha)^2 K^{\frac{1}{3}}T^{\frac{2}{3}}\right)$ when $0 \leq \alpha \leq 1$ and  $1 \leq K \leq T$.
\end{theorem}

\subsection{Proof Overview}
We consider the following families of instances to establish the lower bound.  We first prove the result for $\alpha=0$ and then establish a reduction for $\alpha = \theta$ for $0 \leq \theta < 1$,  to the special case of $\alpha = 0$. 

%\kacomment{What is the significance of the 2 in MAB2? The two objectives? I would say the subscript is kind of understood and we could remove it to make the notation simpler.}
\begin{definition}[Family of instances $\cI$] 
Define a family of instances $\cI$ consisting of instances $\cI^0, \cI^1, \cdots \cI^K$ each with $K+1$ Bernoulli arms indexed by $0,1, \cdots, K$. For the instance $\cI^0$, the costs and mean reward of the $j$-th arm are\\
$ c_j^{\cI^0} =
\begin{cases}
0 & j=0 \\
1 & j\neq 0 
\end{cases}, \hspace{10mm}
\ \mu_j^{\cI^0} =
\begin{cases}
p & j=0 \\
\frac{p}{1-\theta} & j\neq 0 \\
\end{cases}. $ 
for $0 \leq j \leq K$. For the instance $\cI^a$ with $1 \leq a \leq K$, the costs and mean rewards of the $j$-th arm are \\
$ c_j^{\cI^a} =
\begin{cases}
0 & j=0 \\
1 & j\neq 0
\end{cases}, \hspace{5mm}
\ \mu_j^{\cI^a} =
\begin{cases}
p & j= 0 \\
\frac{ p+ \epsilon}{1-\theta} & j=a \\
\frac{p}{1-\theta} & \text{otherwise}  \\
\end{cases}. $ \\
for $0 \leq j \leq K$, where $ 0 \leq \theta <1 ,0<p\leq 1/2, \epsilon >0$ and $( p+ \epsilon)/(1-\theta) < 1$. %\ \epsilon = \frac{p}{2} (\frac{K}{T})^{\frac{1}{3}} $ and $K < T$.
%for some fixed (to be chosen later in the proof) $p, \epsilon, \delta$ with $0 < p \leq 0.5, \  \frac{ p+ \epsilon}{1-\theta} \leq 1- \delta $ and $\epsilon, \delta >0 $. 

%Now, define randomized instance $\rI$ as the instance obtained by randomly choosing from the family of instances $\cI$. $\rI = \cI^0$ with probability $1/2$ and $\rI = \cI^a$ with probability $1/2K$ for $1\leq a \leq K$.

\end{definition}

% \begin{lemma}
% %\begin{restatable}{lemma}{alpha_0_regret}
% \label{lem:alpha_0_neat}
% For any given $K \geq 1, p < 1, T > 0$ and any policy $\pi$ that plays arm $\pi_t$ at time t, we have 
%  $\ex \left[ \qreg(T, 0,\rIzero) + \creg(T, 0,\rIzero) \right]$ is $\Omega\left( pK^{\frac{1}{3}}T^{\frac{2}{3}} \right)$,
%  where the expectation is over the randomization in generating the instance $\rIzero$ and $\epsilon =  \frac{p}{2} (\frac{K}{T})^{\frac{1}{3}}$. %, as well as the random outcomes that result from pulled arms.
% \end{lemma}
% %\end{restatable}

\begin{lemma}\label{lem:alpha_0_neat}
For any given $p, K, T$ and any (possibly randomized) policy $\pi$, there exists an instance $\phi$ (from the family $\cIzero$) such that $\qreg_{\pi}(T, 0,\phi) + \creg_{\pi}(T, 0,\phi)$ is $\Omega\left( pK^{\frac{1}{3}}T^{\frac{2}{3}} \right)$ when $0 < p \leq 1/2$ and $1 \leq  K \leq T$.
\end{lemma}

Lemma \ref{lem:alpha_0_neat} establishes that when $\alpha=0$, any policy must incur a regret of $ \Omega( K^{1/3}T^{2/3})$ on an instance from the family $\cIzero.$ To prove Lemma \ref{lem:alpha_0_neat}, we argue that any online learning algorithm will not be able to differentiate the instance $\cIzero^0$ from the instance $\cIzero^a$ for $1\leq a \leq K$ and therefore, must either incur a high cost regret if the algorithm does not select $0^{\sf th}$ arm frequently or high quality regret if the algorithm selects $0^{\sf th}$ arm frequently. More specifically, any online algorithm would require $O(1/\epsilon^2)$ samples or rounds to distinguish instance $\cIzero^0$ from instance $\cIzero^a$ for $1\leq a \leq K$. Hence, any policy $\pi$ can avoid high quality regret by \textit{exploring sufficiently} for $O(1/\epsilon^2)$ rounds, incurring a cost regret of $O(1/\epsilon^2)$ or incur zero cost regret at the expense of $O(T\epsilon)$ regret on the reward metric. This suggests a trade-off between $1/\epsilon^2$ and $T\epsilon$, which are of the same magnitude at $\epsilon = T^{-1/3}$ resulting in the aforementioned lower bound. The complete proof generalizes techniques from the standard MAB lower bound proof and is provided in Appendix {\ref{sec:lower_bound_appendix}}.
\qed

%Now, for any given $\alpha$, we will show that any online algorithm has to incur a regret of $\Omega((1-\alpha)^2 K^{1/3}T^{2/3})$on instance $I^{\alpha}_{\sf MAB2}$. 
Now, we generalize the above result for $\alpha=0$ to any $\alpha$ for $0 \leq \alpha \leq 1$.
The main idea in our reduction is to show that if there exists an algorithm $\pi_\alpha$ for $\alpha > 0$ such that $\qreg_{\pi}(T, \alpha,\phi) + \creg_{\pi}(T, \alpha,\phi)$ is $o( K^{1/3}T^{2/3})$ on every instance in the family $\cIalpha$, then we can use $\pi_\alpha$ as a subroutine to construct an algorithm $\pi$ for problem \eqref{eq:regret} such that $\qreg_{\pi}(T, 0,\phi) + \creg_{\pi}(T, 0,\phi)$ is $o(K^{1/3}T^{2/3})$ on every instance in $\cIzero$, thus contradicting the lower bound of Lemma \ref{lem:alpha_0_neat}. This will prove
Theorem \ref{thm:lower_bound} by contradiction. To construct the aforementioned sub-routine, we leverage techniques from {\it Bernoulli factory} (\cite{keane1994bernoulli, huber2013nearly}) to generate a sample from a Bernoulli random variable with parameter $\mu/(1-\alpha)$ using samples from a Bernoulli random variable with parameter $\mu$, for any $0 <\mu < 1- \alpha <1.$ We provide the exact sub-routine and complete proof in Appendix \ref{sec:lower_bound_appendix}.
%One of the key tools we use in constructing the algorithm $\algh$ from $\alg$ is \textit{Bernoulli factory} for the linear function. This Bernoulli factory $\bfac (C)$ uses a sequence of IID samples from $Bern(r)$ and returns a sample from $Bern(Cr)$ for any specified $C$. We use the Bernoulli factory described in \cite{huber2013nearly}. This has a guarantee on the expected number of samples $\tau$ from $Bern(r)$ needed to generate a sample from  $Bern(Cr)$. In particular, for a specified $\delta >0$, 

\section{Explore-Then-Commit based algorithm}
\label{sec:exploreFirst}
We propose an explore-then-commit algorithm, named \mainALGfull \ (\mainALG), to have better worst-case performance guarantees as compared to the extensions of the TS and UCB algorithms. As the name suggests, first, this algorithm plays each arm for a specified number of rounds. %After sufficient exploration, the algorithm infers the `best' arm and plays that arm for the remaining rounds. 
After sufficient exploration, the algorithm continues in a UCB-like fashion. In every round, based on the upper and lower confidence bounds on the reward of each arm, a feasible set of arms is constructed as an estimate of all arms having  mean reward greater than the smallest tolerated reward. The lowest cost arm in this feasible set is then pulled. This is detailed in Algorithm \ref{alg:explore_first}.
The key question that arises in this algorithm is how many exploration rounds are needed before exploitation can begin. We establish that $O\left((T/K)^{2/3}\right)$ rounds are sufficient for exploration in the following result (proof in Appendix \ref{sec:alg_performance_appendix}).

% \begin{algorithm}
% \SetAlgoLined
% \KwResult{Arm $I_t$ to be pulled in each round $t \in [T]$ }
% \SetKwInOut{Input}{input}\SetKwInOut{Output}{output}
% \Input{$K, T$}
% \Output{$I_t \ \forall t \in [T]$}
% $T_i(0)= 0 \ \forall i \in [K]$ \\
% Pull each arm $f(T)$ times. \\ 
% Let $\hat{t} = K f(T) + 1.$ \\
% \For{$i \in [K]$}{
%     $\beta_i(\hat{t}) \leftarrow \sqrt{\frac{ \psi^\delta(T_i(\hat{t}-1)) }{T_i(\hat{t}-1)}}$\;
%     $\theta_i(\hat{t}) \leftarrow \hq_i(\hat{t}-1) + \beta_i(\hat{t})$\;
%     $\lambda_i(\hat{t}) \leftarrow \max \{ \hq_i(\hat{t}-1)  - \beta_i(\hat{t}), 0\} $\;
%     }
% $m = \arg \max_{i} \lambda_i(\hat{t})$. \\
% $Feas(\hat{t}) = \{i: \theta_i(\hat{t}) > (1-\alpha) \lambda_m(\hat{t}) \}$. \\
% $I = \arg \min_{i \in Feas(\hat{t})} c_i $. \\
% \For{$t \in [\hat{t}, T]$}{
%     $I_t = I$\;
% }
% \caption{Explore First Algorithm}
% \label{alg:explore_first}
% \end{algorithm}

\begin{algorithm}
\SetAlgoLined
\KwResult{Arm $I_t$ to be pulled in each round $t \in [T]$ }
\SetKwInOut{Input}{Input}\SetKwInOut{Output}{Output}
\Input{$K, T$, no. of exploration pulls per arm $\tau$}
%\Output{$I_t \ \forall t \in [T]$}
$T_i(1)= 0 \ \forall i \in [K]$ \\
\textbf{Pure exploration phase:} \\
\For{$t \in [1, K\tau]$}{
    $I_t = t \mod K$\;
    Pull arm $I_t$ to obtain reward $r_t$\;
    $T_{i}(t+1) = T_i(t) + \mathbf{1}\{ I_t =i \} \ \forall i \in [K]$\;
}
%Let $\hat{t} = K f(T) + 1.$ \\
\textbf{UCB phase:} \\
\For{$t \in [K\tau+1, T]$}{
    $ \hmu_i(t) \leftarrow \left( \sum_{\tau=1}^{t-1} r_{\tau} \ind \{ I_{\tau} =i \} \right)/ { T_i(t) } \ \forall i \in [K] $\;
    $\beta_i(t) \leftarrow \sqrt{ \left(  2 \log{T} \right) /{T_i(t)}} \ \forall i \in [K]$\;
    $\mu^{\sf UCB}_i(t) \leftarrow \min \{ \hmu_i(t) + \beta_i(t), 1 \} \ \forall i \in [K] $\;
    $\mu^{\sf LCB}_i(t) \leftarrow \max \{ \hmu_i(t)  - \beta_i(t), 0\} \ \forall i \in [K] $\;
    $m_t = \arg \max_{i} \mu^{\sf LCB}_i(t)$\;
    $Feas(t) = \{i: \mu^{\sf UCB}_i(t) \geq (1-\alpha) \mu^{\sf LCB}_{m_t}(t) \}$\;
    $I_t = \arg \min_{i \in Feas(t)} c_i $\;
    Pull arm $I_t$ to obtain reward $r_t$\;
    $T_{i}(t+1) = T_i(t) + \mathbf{1}\{ I_t =i \} \ \forall i \in [K]$\;
}
\caption{\mainALGfull}
\label{alg:explore_first}
\end{algorithm}

\begin{theorem}
\label{thm:explore_first_regret}
For an instance $\phi$ with $K$ arms, when  the number of exploration pulls of each arm $ \tau = (T/K)^{2/3}$, then the sum of cost and quality regret incurred by \mainALG \  (Algorithm \ref{alg:explore_first}) on any instance $\phi$ i.e. $\qreg_{\mainALG}(T, \alpha,  \phi) + \creg_{\mainALG}(T, \alpha, \phi)$ is $O(K^{1/3}T^{2/3}\sqrt{\log{T}})$.
\end{theorem}

The key reason that sufficient exploration is needed for our problem is that there can be arms with mean rewards very close to each other but significantly different costs. If cost regret were not of concern, then playing either arm would have led to satisfactory performance by giving low quality regret. The need to perform well on both cost and quality regrets necessitates differentiating between the two arms and finding the one with the cheapest cost among the arms with mean reward above the smallest tolerated reward. 

The regret guarantee mainly stems from the exploration phase of the algorithm. In fact, an algorithm that estimates the optimal arm only once after the exploration phase and pulls that arm for the remaining time will have the same regret upper bound as \mainALG. But we empirically observed that the non-asymptotic performance of this algorithm is worse as compared to Algorithm \ref{alg:explore_first}.

\section{Performance with Constraints on Costs and Rewards}
\label{sec:specialCases}
In this section, we present some extensions of the previous results.

\subsection{Consistent Cost and Quality}
\label{sec:cost_quality_consistent}

The lower bound result in Theorem \ref{thm:lower_bound} is motivated by an extreme instance where arms with very similar mean rewards have very different costs. This raises the following question - can better-performing algorithms be obtained if the rewards and costs are \textit{consistent} with each other? We show that this is indeed the case. Motivated by the instance which led to the worst-case performance, we consider a constraint that gives an upper bound on the difference in costs of every pair of arms by a multiple of the difference in the qualities of these arms. Under this constraint, \consUCB \ has good performance as per the following result with the proof in Appendix \ref{sec:alg_performance_appendix}.

\begin{theorem}
\label{thm:cost_quality_consistent}
If for an instance $\phi$ with $K$ arms,  $|c_i - c_j| \leq \delta |\mu_i - \mu_j| \ \forall i, j \in [K]$ and any (possibly unknown) $\delta >0$, then
$\qreg_{\consUCB}(T, \alpha,  \phi) + \creg_{\consUCB}(T, \alpha, \phi)$ is $O((1+\delta) \sqrt{KT \log{T}})$. 
\end{theorem}

Note that, in general, $\delta$ can be unknown. Hence, even with the above assumption on the consistency of cost and quality, a priori any algorithm cannot get a bound on the quality difference between arms, only by virtue of knowing their costs.
 
\subsection{Unknown Costs}
\label{sec:unknown_costs}
In some applications, the costs of the arms may also be unknown and random. Hence, in addition to the mean reward, the mean costs also need to be estimated. Without loss of generality, we assume that the distribution of the random cost of each arm has support [0,1]. 
Not knowing the arm's cost does not fundamentally change the regret minimization problem we have discussed in the above sections. Clearly, the lower bound result is still valid. 
Algorithm \ref{alg:explore_first} can be generalized to the unknown costs setting with a minor modification in the UCB phase of the algorithm. The modified UCB phase is described in Algorithm \ref{alg:explore_first_unknown_cost}. % in Appendix \ref{sec:unknown_costs_algo_appendix}. %where $\chi_t$ represents the cost observed by pulling the arm in the $t$-th step.
In this algorithm, we maintain confidence bounds on the costs of each arm. Instead of picking the arm with the lowest cost among all feasible arms, the algorithm now picks the arm with the lowest lower confidence bound on cost. Theorem \ref{thm:explore_first_regret} holds for this modified algorithm also. 

Similarly, when costs and quality are consistent as described in Section \ref{sec:cost_quality_consistent}, the \consUCB \ algorithm can be modified to pick the arm with the lowest lower confidence bound on cost and Theorem \ref{thm:cost_quality_consistent} holds.

\begin{algorithm}
\SetAlgoLined
\KwResult{Arm $I_t$ to be pulled in each round $t \in [T]$ }
\SetKwInOut{Input}{input}\SetKwInOut{Output}{output}
\Input{$K, T, $ Number of exploration pulls $\tau$}
%\Output{$I_t \ \forall t \in [T]$}
$T_i(1)= 0 \ \forall i \in [K]$ \\
\textbf{Pure exploration phase:} \\
\For{$t \in [1, Kf(K, T)]$}{
    $I_t = t \mod K$\;
    Pull arm $I_t$ to obtain reward $r_t$ and cost $ \chi_t$\;
    $T_{i}(t+1) = T_i(t) + \mathbf{1}\{ I_t =i \} \ \forall i \in [K]$\;
}
%Let $\hat{t} = K f(T) + 1.$ \\
\textbf{UCB phase:} \\
\For{$t \in [Kf(K, T)+1, T]$}{
    $ \hmu_i(t) \leftarrow \frac{ \sum_{\tau=1}^{t-1} r_{\tau} \ind \{ I_{\tau} =i \}} { T_i(t) } \ \forall i \in [K] $\;
    $ \hc_i(t) \leftarrow \frac{ \sum_{\tau=1}^{t-1} \chi_{\tau} \ind \{ I_{\tau} =i \}} { T_i(t) } \ \forall i \in [K] $\;
    $\beta_i(t) \leftarrow \sqrt{\frac{ 2 \log{T} }{T_i(t)}} \ \forall i \in [K]$\;
    $\mu^{\sf UCB}_i(t) \leftarrow \min \{ \hmu_i(t) + \beta_i(t), 1 \} \ \forall i \in [K] $\;
    $\mu^{\sf LCB}_i(t) \leftarrow \max \{ \hmu_i(t)  - \beta_i(t), 0\} \ \forall i \in [K] $\;
    %$c^{\sf UCB}_i(t) \leftarrow \min \{ \hc_i(t) + \beta_i(t), 1 \} \ \forall i \in [K] $\;
    $c^{\sf LCB}_i(t) \leftarrow \max \{ \hc_i(t)  - \beta_i(t), 0\} \ \forall i \in [K] $\;
    $m_t = \arg \max_{i} \mu^{\sf LCB}_i(t)$\;
    $Feas(t) = \{i: \mu^{\sf UCB}_i(t)> (1-\alpha) \mu^{\sf LCB}_{m_t}(t) \}$\;
    $I_t = \arg \min_{i \in Feas(t)} c^{\sf LCB}_i $\;
    Pull arm $I_t$ to obtain reward $r_t$ and cost $ \chi_t$\;
    $T_{i}(t+1) = T_i(t) + \mathbf{1}\{ I_t =i \} \ \forall i \in [K]$\;
}
\caption{\mainALG \ with Unknown Costs}
\label{alg:explore_first_unknown_cost}
\end{algorithm}

\section{Numerical Experiments}
\label{sec:experiments}
In the previous sections, we have shown theoretical results on the worst-case performance of different algorithms for (\ref{eq:regret}). Now, we illustrate the empirical performance of these algorithms. We shed light on which algorithm performs better in what regime of parameter values.
The key quantity that differentiates the performance of different algorithms is how close the mean rewards of different arms are. We consider a setting with two Bernoulli arms and vary the mean reward of one arm (the cheaper arm) while keeping the other quantities (reward distribution of the other arm and costs of both arms) fixed. The values of these parameters are described in Table \ref{tab:param_values}. The reward in each round follows a Bernoulli distribution, whereas the cost is a known fixed value. The cost and quality regret at time $T$ of the different algorithms are plotted in Figure \ref{fig:varying_qual_perf}.

%\begin{center}
\begin{table}[]
    \centering
    \begin{tabular}{ |c|c| } \hline
    \textbf{ Parameter} & \textbf{ Value} \\ \hline
    Mean reward of arm 1 $(\mu_{1})$ & 0.5 \\ \hline
    Mean reward of arm 2 $(\mu_{2})$ & 0.3-0.6 \\ \hline
    Cost of arm 1 $(c_1)$ & 1 \\  \hline
    Cost of arm 2 $(c_2)$ & 0 \\ \hline
    Subsidy factor $(\alpha)$ & 0.1 \\ \hline
    Time horizon $(T)$ & 5000 \\ \hline
    \end{tabular}
    \caption{Parameter values}
    \label{tab:param_values}
%\end{center}
\end{table}

We observe that the performance of \consTS \ and \consUCB \ are close to each other for the entire range of mean reward values. To compare these algorithms' performance with \mainALG, we focus on how close the mean reward of the lower mean reward arm is to the smallest tolerated reward. When $\mu_2 \leq 0.5$ ($\mu_2 > 0.5$), the lowest tolerated reward is 0.45 ($0.9 \mu_2$). In terms of quality regret, when $\mu_2$ is much smaller than 0.45, \consTS \ and \consUCB \ perform much better than \mainALG. This is because the number of exploration rounds in the \mainALG \ algorithm is fixed (independent of the difference in mean rewards of the two arms), leading to higher quality regret when $\mu_2$ is much smaller than 0.45. On the other hand, because of the large difference in $\mu_2$ and 0.45, \consTS \ and \consUCB \ algorithms can easily find the optimal arm and incur low quality regret. The cost regret of all algorithms is 0 because the optimal arm is the expensive arm. 

When $\mu_2$ is close to (and less than) 0.45, \consTS \ and \consUCB \ incur much higher cost regret as compared to \mainALG. This is in line with the intuition established in Section \ref{sec:ts_linear_regret}. Here, \consTS \ and \consUCB \ are unable to effectively conclude that the second (cheaper) arm is optimal. Thus, they pull the first (expensive) arm many times, leading to high cost regret. On the other hand, \mainALG,   after the exploration rounds, can correctly identify the second arm as the optimal arm. 

\begin{figure}
    \centering
    \includegraphics[width=0.35\textwidth]{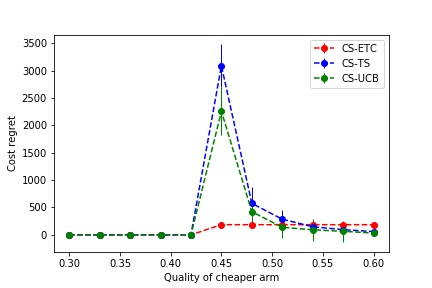}
     \includegraphics[width=0.35\textwidth]{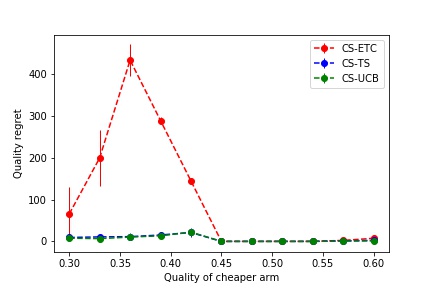}
    \caption{Performance of algorithms with varying mean reward of the cheaper arm. The length of the error bars correspond to two standard deviations in regret obtained by running the experiment 50 times.}
    \label{fig:varying_qual_perf}
\end{figure}

Thus, we recommend using the \consTS /\consUCB \ algorithm when the mean rewards of arms are well-differentiated and \mainALG \  when the mean rewards are close to one another (as is often the case in the SMS application). This is in line with the notion that algorithms that perform well in the worst case might not have good performance for an average case.

%If the difference in the qualities of the arms is kept constant, varying the difference in costs of the arms does not qualitatively change the performance of the different algorithms. Thus, when costs of all the arms are known, the regime in which each algorithm performs well is mainly determined by the difference in qualities of the arms. 

\subsection{Conclusion and Future Work}
In this paper, we have proposed a new variant of the \MAB  problem, which factors costs associated with playing an arm and introduces new metrics that uniquely capture the features of multiple real-world applications. We argue about the \textit{hardness} of this problem by establishing fundamental limits on the performance of any online algorithm and also demonstrating that traditional \MAB  algorithms perform poorly from both a theoretical and empirical standpoint. We present a simple near-optimal algorithm, and through numerical simulations, we prescribe an ideal algorithmic choice for different problem regimes. 

An interesting direction for future work is defining a notion of instance-wise optimality and devising algorithms that achieve this. Insights from this can also help improve the \mainALG \ algorithm to obtain better empirical performance.
Another important question that naturally arises from this work is developing an algorithm for the adversarial variant of the \MAB  with cost subsidy problem. In particular, it is not immediately clear if EXP3 (\cite{auer2002nonstochastic}) family of algorithms, which are popular for non-stochastic \MAB problem, can be generalized to settings where the reward distribution is not stationary.

\subsubsection*{Acknowledgements}
We want to thank Nicolas Stier-Moses for introducing us to the SMS routing application and helping us in formulating it as an \MAB problem. We would also like to thank three anonymous reviewers for their detailed feedback, which sharpened our arguments.

\bibliographystyle{abbrvnat}
\bibliography{references}

\begin{thebibliography}{52}
\providecommand{\natexlab}[1]{#1}
\providecommand{\url}[1]{\texttt{#1}}
\expandafter\ifx\csname urlstyle\endcsname\relax
  \providecommand{\doi}[1]{doi: #1}\else
  \providecommand{\doi}{doi: \begingroup \urlstyle{rm}\Url}\fi

\bibitem[Abramowitz and Stegun(1948)]{abramowitz1948handbook}
M.~Abramowitz and I.~A. Stegun.
\newblock \emph{Handbook of mathematical functions with formulas, graphs, and
  mathematical tables}, volume~55.
\newblock US Government printing office, 1948.

\bibitem[Agrawal and Devanur(2014)]{AD14}
S.~Agrawal and N.~R. Devanur.
\newblock Bandits with concave rewards and convex knapsacks.
\newblock In \emph{Proceedings of the Fifteenth ACM Conference on Economics and
  Computation}, EC ’14, page 989–1006, New York, NY, USA, 2014. Association
  for Computing Machinery.
\newblock ISBN 9781450325653.

\bibitem[Agrawal and Goyal(2017{\natexlab{a}})]{AgrawalTS_nearopt}
S.~Agrawal and N.~Goyal.
\newblock Near-optimal regret bounds for thompson sampling.
\newblock \emph{J. ACM}, 64\penalty0 (5), 2017{\natexlab{a}}.

\bibitem[Agrawal and Goyal(2017{\natexlab{b}})]{agrawal2017near}
S.~Agrawal and N.~Goyal.
\newblock Near-optimal regret bounds for thompson sampling.
\newblock \emph{Journal of the ACM (JACM)}, 64\penalty0 (5):\penalty0 1--24,
  2017{\natexlab{b}}.

\bibitem[Agrawal et~al.(2016)Agrawal, Avadhanula, Goyal, and
  Zeevi]{agrawalnear}
S.~Agrawal, V.~Avadhanula, V.~Goyal, and A.~Zeevi.
\newblock A near-optimal exploration-exploitation approach for assortment
  selection.
\newblock \emph{Proceedings of the 2016 ACM Conference on Economics and
  Computation (EC)}, pages 599--600, 2016.

\bibitem[Agrawal et~al.(2017)Agrawal, Avadhanula, Goyal, and
  Zeevi]{agrawal2017thompson}
S.~Agrawal, V.~Avadhanula, V.~Goyal, and A.~Zeevi.
\newblock Thompson sampling for the mnl-bandit.
\newblock In \emph{Conference on Learning Theory}, pages 76--78, 2017.

\bibitem[Amani et~al.(2020)Amani, Alizadeh, and
  Thrampoulidis]{amani2020generalized}
S.~Amani, M.~Alizadeh, and C.~Thrampoulidis.
\newblock Generalized linear bandits with safety constraints.
\newblock In \emph{ICASSP 2020-2020 IEEE International Conference on Acoustics,
  Speech and Signal Processing (ICASSP)}, pages 3562--3566. IEEE, 2020.

\bibitem[Amazon(2019)]{Amazon}
Amazon.
\newblock Amazon auto-targeting, 2019.
\newblock URL \url{https://tinyurl.com/yx9lyfwq}.

\bibitem[Anandkumar et~al.(2011)Anandkumar, Michael, Tang, and
  Swami]{anandkumar2011distributed}
A.~Anandkumar, N.~Michael, A.~K. Tang, and A.~Swami.
\newblock Distributed algorithms for learning and cognitive medium access with
  logarithmic regret.
\newblock \emph{IEEE Journal on Selected Areas in Communications}, 29\penalty0
  (4):\penalty0 731--745, 2011.

\bibitem[Auer et~al.(2002{\natexlab{a}})Auer, Cesa-Bianchi, Freund, and
  Schapire]{Auer2002}
P.~Auer, N.~Cesa-Bianchi, Y.~Freund, and R.~E. Schapire.
\newblock The nonstochastic multiarmed bandit problem.
\newblock \emph{SIAM Journal on Computing}, 32\penalty0 (1):\penalty0 48--77,
  2002{\natexlab{a}}.

\bibitem[Auer et~al.(2002{\natexlab{b}})Auer, Cesa-Bianchi, Freund, and
  Schapire]{auer2002nonstochastic}
P.~Auer, N.~Cesa-Bianchi, Y.~Freund, and R.~E. Schapire.
\newblock The nonstochastic multiarmed bandit problem.
\newblock \emph{SIAM journal on computing}, 32\penalty0 (1):\penalty0 48--77,
  2002{\natexlab{b}}.

\bibitem[Badanidiyuru et~al.(2018)Badanidiyuru, Kleinberg, and
  Slivkins]{Badanidiyuru:2018}
A.~Badanidiyuru, R.~Kleinberg, and A.~Slivkins.
\newblock Bandits with knapsacks.
\newblock \emph{J. ACM}, 65\penalty0 (3):\penalty0 13:1--13:55, Mar. 2018.
\newblock ISSN 0004-5411.
\newblock \doi{10.1145/3164539}.
\newblock URL \url{http://doi.acm.org/10.1145/3164539}.

\bibitem[Bubeck and Cesa-Bianchi(2012)]{MAL-024}
S.~Bubeck and N.~Cesa-Bianchi.
\newblock Regret analysis of stochastic and nonstochastic multi-armed bandit
  problems.
\newblock \emph{Foundations and Trends® in Machine Learning}, 5\penalty0
  (1):\penalty0 1--122, 2012.

\bibitem[Canlas et~al.(2010)Canlas, Cruz, Dimarucut, Uyengco, Tangonan, Guico,
  Libatique, and Pineda]{canlas2010quantitative}
M.~Canlas, K.~P. Cruz, M.~K. Dimarucut, P.~Uyengco, G.~Tangonan, M.~L. Guico,
  N.~Libatique, and C.~Pineda.
\newblock A quantitative analysis of the quality of service of short message
  service in the philippines.
\newblock In \emph{2010 IEEE International Conference on Communication
  Systems}, pages 710--714. IEEE, 2010.

\bibitem[Cao et~al.(2015)Cao, Li, Tao, and Li]{cao2015top}
W.~Cao, J.~Li, Y.~Tao, and Z.~Li.
\newblock On top-k selection in multi-armed bandits and hidden bipartite
  graphs.
\newblock In \emph{Advances in Neural Information Processing Systems}, pages
  1036--1044, 2015.

\bibitem[Cesa-Bianchi et~al.(2014)Cesa-Bianchi, Gentile, and
  Mansour]{cesa2014regret}
N.~Cesa-Bianchi, C.~Gentile, and Y.~Mansour.
\newblock Regret minimization for reserve prices in second-price auctions.
\newblock \emph{IEEE Transactions on Information Theory}, 61\penalty0
  (1):\penalty0 549--564, 2014.

\bibitem[Chen et~al.(2016)Chen, Gupta, and Li]{chen2016pure}
L.~Chen, A.~Gupta, and J.~Li.
\newblock Pure exploration of multi-armed bandit under matroid constraints.
\newblock In \emph{Conference on Learning Theory}, pages 647--669, 2016.

\bibitem[Chen et~al.(2014)Chen, Lin, King, Lyu, and
  Chen]{chen2014combinatorial}
S.~Chen, T.~Lin, I.~King, M.~R. Lyu, and W.~Chen.
\newblock Combinatorial pure exploration of multi-armed bandits.
\newblock In \emph{Advances in Neural Information Processing Systems}, pages
  379--387, 2014.

\bibitem[Daulton et~al.(2019)Daulton, Singh, Avadhanula, Dimmery, and
  Bakshy]{daulton2019thompson}
S.~Daulton, S.~Singh, V.~Avadhanula, D.~Dimmery, and E.~Bakshy.
\newblock Thompson sampling for contextual bandit problems with auxiliary
  safety constraints.
\newblock \emph{arXiv preprint arXiv:1911.00638}, 2019.

\bibitem[Drugan and Nowe(2013)]{drugan2013designing}
M.~M. Drugan and A.~Nowe.
\newblock Designing multi-objective multi-armed bandits algorithms: A study.
\newblock In \emph{The 2013 International Joint Conference on Neural Networks
  (IJCNN)}, pages 1--8. IEEE, 2013.

\bibitem[Facebook(2016)]{Facebook}
Facebook.
\newblock Facebook targeting expansion, 2016.
\newblock URL \url{https://tinyurl.com/y3ss2j8g}.

\bibitem[Galichet et~al.(2013)Galichet, Sebag, and
  Teytaud]{galichet2013exploration}
N.~Galichet, M.~Sebag, and O.~Teytaud.
\newblock Exploration vs exploitation vs safety: Risk-aware multi-armed
  bandits.
\newblock In \emph{Asian Conference on Machine Learning}, pages 245--260, 2013.

\bibitem[Google(2014)]{Google}
Google.
\newblock Google auto-targeting, 2014.
\newblock URL \url{https://tinyurl.com/y3c4bdaj}.

\bibitem[Huber(2013)]{huber2013nearly}
M.~Huber.
\newblock Nearly optimal bernoulli factories for linear functions.
\newblock \emph{arXiv preprint arXiv:1308.1562}, 2013.

\bibitem[{Immorlica} et~al.(2019){Immorlica}, {Sankararaman}, {Schapire}, and
  {Slivkins}]{Immorlica19}
N.~{Immorlica}, K.~A. {Sankararaman}, R.~{Schapire}, and A.~{Slivkins}.
\newblock Adversarial bandits with knapsacks.
\newblock In \emph{2019 IEEE 60th Annual Symposium on Foundations of Computer
  Science (FOCS)}, pages 202--219, 2019.

\bibitem[Jamieson and Nowak(2014)]{jamieson2014best}
K.~Jamieson and R.~Nowak.
\newblock Best-arm identification algorithms for multi-armed bandits in the
  fixed confidence setting.
\newblock In \emph{2014 48th Annual Conference on Information Sciences and
  Systems (CISS)}, pages 1--6. IEEE, 2014.

\bibitem[Katz-Samuels and Scott(2019)]{katz2019top}
J.~Katz-Samuels and C.~Scott.
\newblock Top feasible arm identification.
\newblock In \emph{The 22nd International Conference on Artificial Intelligence
  and Statistics}, pages 1593--1601, 2019.

\bibitem[Kazerouni et~al.(2017)Kazerouni, Ghavamzadeh, Yadkori, and
  Van~Roy]{kazerouni2017conservative}
A.~Kazerouni, M.~Ghavamzadeh, Y.~A. Yadkori, and B.~Van~Roy.
\newblock Conservative contextual linear bandits.
\newblock In \emph{Advances in Neural Information Processing Systems}, pages
  3910--3919, 2017.

\bibitem[Keane and O'Brien(1994)]{keane1994bernoulli}
M.~Keane and G.~L. O'Brien.
\newblock A bernoulli factory.
\newblock \emph{ACM Transactions on Modeling and Computer Simulation (TOMACS)},
  4\penalty0 (2):\penalty0 213--219, 1994.

\bibitem[Koningstein(2006)]{koningstein2006suggesting}
R.~Koningstein.
\newblock Suggesting and/or providing targeting information for advertisements,
  July~6 2006.
\newblock US Patent App. 11/026,508.

\bibitem[Langford and Zhang(2008)]{langford2008epoch}
J.~Langford and T.~Zhang.
\newblock The epoch-greedy algorithm for multi-armed bandits with side
  information.
\newblock In \emph{Advances in neural information processing systems}, pages
  817--824, 2008.

\bibitem[Li et~al.(2011)Li, Chu, Langford, and Wang]{li2011unbiased}
L.~Li, W.~Chu, J.~Langford, and X.~Wang.
\newblock Unbiased offline evaluation of contextual-bandit-based news article
  recommendation algorithms.
\newblock In \emph{Proceedings of the fourth ACM international conference on
  Web search and data mining}, pages 297--306, 2011.

\bibitem[Li et~al.(2015)Li, Chen, Kleban, and Gupta]{li2015counterfactual}
L.~Li, S.~Chen, J.~Kleban, and A.~Gupta.
\newblock Counterfactual estimation and optimization of click metrics in search
  engines: A case study.
\newblock In \emph{Proceedings of the 24th International Conference on World
  Wide Web}, pages 929--934, 2015.

\bibitem[MarketWatch(2020)]{MarketWatch}
MarketWatch.
\newblock Marketwatch a2p report, 2020.
\newblock URL \url{https://rb.gy/0w96oi}.

\bibitem[May et~al.(2012)May, Korda, Lee, and Leslie]{may2012optimistic}
B.~C. May, N.~Korda, A.~Lee, and D.~S. Leslie.
\newblock Optimistic bayesian sampling in contextual-bandit problems.
\newblock \emph{Journal of Machine Learning Research}, (13):\penalty0
  2069--2106, 2012.

\bibitem[Meng et~al.(2007)Meng, Zerfos, Samanta, Wong, and
  Lu]{meng2007analysis}
X.~Meng, P.~Zerfos, V.~Samanta, S.~H. Wong, and S.~Lu.
\newblock Analysis of the reliability of a nationwide short message service.
\newblock In \emph{IEEE INFOCOM 2007-26th IEEE International Conference on
  Computer Communications}, pages 1811--1819. IEEE, 2007.

\bibitem[Oliver and Li(2011)]{chapelle}
C.~Oliver and L.~Li.
\newblock An empirical evaluation of thompson sampling.
\newblock \emph{In Advances in Neural Information Processing Systems (NIPS)},
  24:\penalty0 2249?2257, 2011.

\bibitem[Osunade and Nurudeen()]{osunaderoute}
O.~Osunade and S.~O. Nurudeen.
\newblock Route optimization for delivery of short message service in
  telecommunication networks.

\bibitem[Ovum(2017)]{Ovum}
Ovum.
\newblock Sustaining a2p sms growth while securing mobile network, 2017.
\newblock URL \url{https://rb.gy/qqonzd}.

\bibitem[Paria et~al.(2018)Paria, Kandasamy, and P{\'o}czos]{paria2018flexible}
B.~Paria, K.~Kandasamy, and B.~P{\'o}czos.
\newblock A flexible framework for multi-objective bayesian optimization using
  random scalarizations.
\newblock \emph{arXiv preprint arXiv:1805.12168}, 2018.

\bibitem[Robbins(1952)]{robbins1952some}
H.~Robbins.
\newblock Some aspects of the sequential design of experiments.
\newblock \emph{Bulletin of the American Mathematical Society}, 58\penalty0
  (5):\penalty0 527--535, 1952.

\bibitem[Sankararaman et~al.(2019)Sankararaman, Ganesh, and
  Shakkottai]{sankararaman2019social}
A.~Sankararaman, A.~Ganesh, and S.~Shakkottai.
\newblock Social learning in multi agent multi armed bandits.
\newblock \emph{Proceedings of the ACM on Measurement and Analysis of Computing
  Systems}, 3\penalty0 (3):\penalty0 1--35, 2019.

\bibitem[Scott(2010)]{scott2010modern}
S.~L. Scott.
\newblock A modern bayesian look at the multi-armed bandit.
\newblock \emph{Applied Stochastic Models in Business and Industry},
  26\penalty0 (6):\penalty0 639--658, 2010.

\bibitem[Slivkins(2019)]{MAL-068}
A.~Slivkins.
\newblock Introduction to multi-armed bandits.
\newblock \emph{Foundations and Trends® in Machine Learning}, 12\penalty0
  (1-2):\penalty0 1--286, 2019.
\newblock ISSN 1935-8237.
\newblock URL \url{http://dx.doi.org/10.1561/2200000068}.

\bibitem[Slivkins and Vaughan(2014)]{slivkins2014online}
A.~Slivkins and J.~W. Vaughan.
\newblock Online decision making in crowdsourcing markets: Theoretical
  challenges.
\newblock \emph{ACM SIGecom Exchanges}, 12\penalty0 (2):\penalty0 4--23, 2014.

\bibitem[Thompson(1933)]{thompson}
W.~Thompson.
\newblock On the likelihood that one unknown probability exceeds another in
  view of the evidence of two samples.
\newblock \emph{Biometrika}, 25\penalty0 (3/4):\penalty0 285--294, 1933.

\bibitem[Tran-Thanh et~al.(2012)Tran-Thanh, Chapman, Rogers, and
  Jennings]{tran2012knapsack}
L.~Tran-Thanh, A.~Chapman, A.~Rogers, and N.~R. Jennings.
\newblock Knapsack based optimal policies for budget--limited multi--armed
  bandits.
\newblock In \emph{Twenty-Sixth AAAI Conference on Artificial Intelligence},
  2012.

\bibitem[Twilio and Uber(2020)]{twiliober}
Twilio and Uber.
\newblock Uber built a great ridesharing experience with sms \& voice, 2020.
\newblock URL \url{https://customers.twilio.com/208/uber/}.

\bibitem[Wu et~al.(2016)Wu, Shariff, Lattimore, and
  Szepesv{\'a}ri]{wu2016conservative}
Y.~Wu, R.~Shariff, T.~Lattimore, and C.~Szepesv{\'a}ri.
\newblock Conservative bandits.
\newblock In \emph{International Conference on Machine Learning}, pages
  1254--1262, 2016.

\bibitem[Yahyaa and Manderick(2015)]{yahyaa2015thompson}
S.~Yahyaa and B.~Manderick.
\newblock Thompson sampling for multi-objective multi-armed bandits problem.
\newblock In \emph{Proceedings}, page~47. Presses universitaires de Louvain,
  2015.

\bibitem[Yahyaa et~al.(2014)Yahyaa, Drugan, and Manderick]{yahyaa2014annealing}
S.~Q. Yahyaa, M.~M. Drugan, and B.~Manderick.
\newblock Annealing-pareto multi-objective multi-armed bandit algorithm.
\newblock In \emph{2014 IEEE Symposium on Adaptive Dynamic Programming and
  Reinforcement Learning (ADPRL)}, pages 1--8. IEEE, 2014.

\bibitem[Zerfos et~al.(2006)Zerfos, Meng, Wong, Samanta, and
  Lu]{zerfos2006study}
P.~Zerfos, X.~Meng, S.~H. Wong, V.~Samanta, and S.~Lu.
\newblock A study of the short message service of a nationwide cellular
  network.
\newblock In \emph{Proceedings of the 6th ACM SIGCOMM conference on Internet
  measurement}, pages 263--268, 2006.

\end{thebibliography}

\appendix
\onecolumn

%\aistatstitle{Multi-armed Bandits with Cost Subsidy:\\
%Supplementary Material}

\begin{center}
    \Large{\textbf{Multi-armed Bandits with Cost Subsidy:\\
Supplementary Material}}\\
\end{center}

\paragraph{Outline}
The supplementary material of the paper is organized as follows.

\begin{itemize}
    \item Appendix A contains technical lemmas used in subsequent proofs.
    \item Appendix B contains a proof of the lower bound.
    \item Appendix C contains proofs related to the performance of various algorithms presented in the paper.
    \item Appendix D gives a detailed description of the \mainALG algorithm when the costs of the arms are unknown and random.
\end{itemize}

\section{Technical Lemmas}

\begin{lemma}[Taylor's Series Approximation] For $x>0, \ \ln(1+x) \geq x - \frac{x^2}{1-x^2}$.
\label{lem:logineq1}
\end{lemma}
\begin{proof} For $x>0$,
\begin{align*}
    \ln(1+x) &= x - \frac{x^2}{2} + \frac{x^3}{3} - \frac{x^4}{4} + \cdots \\
    & \geq x - \frac{x^2}{2} - \frac{x^4}{4} - \cdots  \hspace{10mm} (\text{because } x>0) \\
    &\geq x - x^2 -x^4 \\
    & = x - x^2 (1 + x^2 +x^4 + \cdots) \\
    &= x- \frac{x^2}{1-x^2}.
\end{align*}
\end{proof}

\begin{lemma}[Taylor's Series Approximation] For $x>0, \ \ln(1-x) \geq -x - \frac{x^2}{1-x}$.
\label{lem:logineq2}
\end{lemma}
\begin{proof}
For $x>0$,
\begin{align*}
    \ln(1-x) &= -x - \frac{x^2}{2} - \frac{x^3}{3} - \frac{x^4}{4} + \cdots \\
    & \geq -x - x^2 -x^3 -x^4 - \cdots  \hspace{10mm} (\text{because } x>0) \\
    &= -x - x^2(1+x+x^2+\cdots) \\
    &= -x- \frac{x^2}{1-x}.
\end{align*}
\end{proof}

\begin{lemma}[Pinsker's inequality]
\label{lem:klbound}
Let $\ber(x)$ denote a Bernoulli distribution with mean $x$ where $0 \leq x \leq 1$. Then, $KL(\ber(p);\ber(p + \epsilon)) \leq \frac{4\epsilon^2}{ p  }$ where $0 < p \leq  \frac{1}{2}$, $0 < \epsilon \leq \frac{p}{2} $ and $ p + \epsilon <1 $ and  the KL divergence between two Bernoulli distributions with mean $x$ and $y$ is given as $KL(\ber(x);\ber(y)) = x \ln{\frac{x}{y}} + (1-x) \ln{\frac{1-x}{1-y}} $.
\end{lemma}

\begin{proof}

$   KL(\ber(p);\ber(p + \epsilon)) = p \ln{\frac{p}{p + \epsilon}} + (1-p) \ln{\frac{1-p}{1-p -\epsilon}} $

\begin{align*}
    KL(\ber(p);\ber(p + \epsilon)) & = p \ln{\frac{p}{p + \epsilon}} + (1-p) \ln{\frac{1-p}{1-p -\epsilon}}  \\
   &= -p \ln\left( 1 + \frac{\epsilon}{p} \right) - (1-p)\ln\left(1 - \frac{\epsilon}{1-p} \right) \\
   &\leq - p \left( \frac{\epsilon}{p} - \frac{\frac{\epsilon^2}{p^2} }{1 - \frac{\epsilon^2}{p^2}} \right) - (1-p) \left( -\frac{\epsilon}{1-p} - \frac{\frac{\epsilon^2}{(1-p)^2} }{1 - \frac{\epsilon}{1-p}} \right) \\
   & \hspace{5mm} (\text{Using Lemmas \ref{lem:logineq1} and \ref{lem:logineq2}. }) \\
  \end{align*}
  Thus, \begin{align*}
    KL(\ber(p);\ber(p + \epsilon)) 
   & \leq -\epsilon + \frac{\epsilon^2}{ p \left( 1 - \frac{\epsilon^2}{p^2} \right)  } + \epsilon + \frac{\epsilon^2}{(1-p) \left( 1-\frac{\epsilon}{1-p} \right) } \\
   & \leq \frac{\epsilon^2}{ p \left( 1- \frac{1}{4} \right) } + \frac{\epsilon^2}{ (1-p) \left( 1- \frac{1}{2} \right)}  \hspace{10mm} ( \text{because } \frac{\epsilon}{1-p} \leq \frac{\epsilon}{p} \leq \frac{1}{2}  ) \\
   &= \frac{4\epsilon^2}{3p} + \frac{2 \epsilon^2}{1-p} \\
   & \leq \frac{2\epsilon^2}{p} + \frac{2 \epsilon^2}{1-p} \\
   & \leq \frac{2\epsilon^2}{p}  + \frac{2\epsilon^2}{p} \hspace{10mm} (\text{because } p \leq \frac{1}{2}) \\
   &= \frac{4 \epsilon^2}{ p }   .
\end{align*}

%The last inequality follows from the fact that the product $(1-p - \epsilon)(p + \epsilon)$ is increasing in $\epsilon$.

\end{proof}

\section{Proof of Lower Bound}
\label{sec:lower_bound_appendix}

% \begin{lemma}\label{lem:alpha_0_neat}
% For any given $K \geq 1, p < 1, T > 0$ and any policy $\pi$ that plays arm $\pi_t$ at time t, we have 
% %\begin{multline*}
%  $\ex \left[ \qreg(T, 0,\rIzero) + \creg(T, 0,\rIzero) \right]$ is $\Omega\left( pK^{\frac{1}{3}}T^{\frac{2}{3}} \right)$,
%  %\end{multline*}
%  where the expectation is over the randomization in generating the instance $\rIzero$ and $\epsilon =  \frac{p}{2} (\frac{K}{T})^{\frac{1}{3}}$. %, as well as the random outcomes that result from pulled arms.
% \end{lemma}
% \begin{lemma}\label{lem:alpha_0_neat}
% For any given $p, K, T$ and any (possibly randomized) policy $\pi$, there exists an instance $\phi$ such that $\qreg_{\pi}(T, 0,\phi) + \creg_{\pi}(T, 0,\phi)$ is $\Omega\left( pK^{\frac{1}{3}}T^{\frac{2}{3}} \right)$ when $0 < p \leq 1/2$ and $1 \leq  K \leq T$.
% \end{lemma}

\begin{proof} [Proof of Lemma \ref{lem:alpha_0_neat}]
In the family of instances $\cI$, the costs of the arms are same across instances. Arm 0 is the cheapest arm in all the instances. With this, we define a modified notion of quality regret which penalizes the regret only when this cheap arm is pulled as
\begin{equation}
 \mqreg_{\pi}(T, \alpha,  \bmu, \bc) = \sum_{t=1}^T \max \{ \mu_{m*} - \mu_{\pi_t}, 0 \}  \ind(c_{i_t}=0).
\end{equation}
An equivalent notation for denoting the modified regret of policy $\pi$ on an instance $I$ of the problem is $\mqreg_\pi(T, \alpha,I)$.
This modified quality regret is at most equal to the quality regret. For proving the lemma, we will show a stronger result that  there exists an instance $\rIzero$ such that $\mqreg(T, 0,\rIzero) + \creg(T, 0,\rIzero)$ is $\Omega\left( pK^{\frac{1}{3}}T^{\frac{2}{3}} \right)$ which will imply the required result.

% Let us denote a Bernoulli distribution with mean $p$ as $Ber(p)$. Using Lemma \ref{lem:klbound}, we know that the KL divergence between the Bernoulli distributions with mean $\epsilon$ and $\epsilon + \theta$ is upper-bounded by $\frac{4 \epsilon^2}{\theta}$ i.e. $KL(Ber(p); Ber(p+\epsilon)) \leq \frac{4 \epsilon^2}{\theta}$ (with KL divergence defined in terms of the natural logarithm). 

Let us first consider any deterministic policy (or algorithm) $\pi$.
For a deterministic algorithm, the number of times an arm is pulled is a function of the observed rewards. 
Let the number of times arm $j$ is played be denoted by $N_j$ and let the total number of times any arm with cost 1 i.e. an expensive arm is played be $N_{exp} = 1 - N_0$. For any $a$ such that $1\leq a \leq K$, we can use the proof of Lemma A.1 in \cite{auer2002nonstochastic}, with function $f(\textbf{r}) = N_{exp}$ to get
\begin{align*}
\ex^a\left[N_{exp}\right] \leq  \ex^0\left[N_{exp}\right] + 0.5T \sqrt{ 2 \ex^0[N_a] KL ( Ber(p); Ber(p+\epsilon) ) }
\end{align*}

where $\ex^j$ is the expectation operator with respect to the probability distribution defined by the random rewards in instance $\cIzero^j$. Thus, using Lemma \ref{lem:klbound}, we get,
\begin{equation}
\label{eqn:lemmaA1_3}  
\ex^a\left[N_{exp}\right] \leq  \ex^0\left[N_{exp}\right] + 0.5T \sqrt{ \ex^0[N_a] 8 \epsilon^2/{p} }.
\end{equation}

Now, let us look at the regret of the algorithm for each instance in the family $\cIzero$. We have 
\begin{enumerate}
    \item $ \creg_{\pi}(T, \alpha, \cIzero^0 ) = \ex^0[ N_{exp}], \ \mqreg_{\pi}(T, \alpha, \cIzero^0 ) = 0 $ 
    \item $ \creg_{\pi}(T, \alpha, \cIzero^a ) = 0, \ \mqreg_{\pi}(T, \alpha, \cIzero^a )  = \epsilon \left( T - \ex^a[N_{exp}]  \right) $.
\end{enumerate}

Now, define randomized instance $\rIzero$ as the instance obtained by randomly choosing from the family of instances $\cIzero$ such that $\rIzero = \cIzero^0$ with probability $1/2$ and $\rIzero = \cIzero^a$ with probability $1/2K$ for $1\leq a \leq K$. The expected regret of this randomized instance is

\begin{align*}
    &\ex \left[ \mqreg_{\pi}(T, 0,\rIzero) + \creg_{\pi}(T, 0,\rIzero) \right] \\
    & = \frac{1}{2} \left( \mqreg_{\pi}(T, \alpha, \cIzero^0 ) + \creg_{\pi}(T, \alpha, \cIzero^0 ) \right) + \\
    & \ \ \frac{1}{2K} \sum_{a=1}^K \left( \mqreg_{\pi}(T, \alpha, \cIzero^a )  +  \creg_{\pi}(T, \alpha, \cIzero^a )  \right)  \\
    & = \frac{1}{2}  \ex^0 [N_{exp}] + \frac{1}{2K} \sum_{a=1}^K \epsilon (T- \ex^a[N_{exp}]) \\
    & \geq  \frac{1}{2}  \ex^0 [N_{exp}] +  \frac{1}{2K} \sum_{a=1}^K  \epsilon \left(T-   \ex^0\left[N_{exp}\right] - \frac{1}{2}T \sqrt{\ex^0[N_a]  \frac{8 \epsilon^2}{p} } \right) \hspace{5mm} (\text{using } (\ref{eqn:lemmaA1_3}) ) \\
    & = \frac{1}{2} \left[ \epsilon T + (1-\epsilon) \sum_{a=1}^K \ex ^0[N_a] - \frac{T \epsilon}{2K} \sum_{a=1}^K \sqrt{\frac{8 \epsilon^2}{p} \ex^0[N_a] } \right] \\
    &= \frac{1}{2} \sum_{a=1}^K \left[ \frac{\epsilon T}{K} + (1-\epsilon) (\ex ^0[N_a])^2 -T \ex ^0[N_a] \epsilon^2 \frac{\sqrt{2}}{K \sqrt{p}} \right] \\
    &= \frac{1}{2} \sum_{a=1}^K  \left[ \left( \sqrt{1-\epsilon} \ex ^0[N_a] - \frac{\epsilon^2 T}{2K } \sqrt{\frac{2}{p(1-\epsilon)}}  \right)^2 + \frac{\epsilon T}{K} - \frac{\epsilon^4 T^2}{2 p K^2 (1-\epsilon)} \right] \\
    &\geq \frac{1}{2} \sum_{a=1}^K \frac{\epsilon T}{K} - \frac{ \epsilon^4 T^2}{2p K^2 (1-\epsilon)} \\
    &= \frac{\epsilon T}{2} - \frac{\epsilon^4 T^2}{4pK (1-\epsilon)}
\end{align*}

Taking $\epsilon = \frac{p}{2} (\frac{K}{T})^{\frac{1}{3}}$, we get $\ex \left[ \mqreg_{\pi}(T, 0,\rIzero) + \creg_{\pi}(T, 0,\rIzero) \right]$ is $\Omega( p K^{1/3}T^{2/3})$ when $K \leq T$.

Using Yao's principle, for any randomized algorithm $\pi$, there exists an instance $\cIzero^j$ with  $0\leq j \leq K$ such that   $\mqreg_{\pi}(T, 0,\cIzero^j) + \creg_{\pi}(T, 0,\cIzero^j)$ is  $\Omega( p K^{1/3}T^{2/3})$. Also, since  $ \mqreg_{\pi}(T, 0,\cIzero^j)  \leq \qreg_{\pi}(T, 0,\cIzero^j) $, we have $\qreg_{\pi}(T, 0,\cIzero^j) + \creg_{\pi}(T, 0,\cIzero^j)$ is $\Omega( p K^{1/3}T^{2/3})$.
\end{proof}

% \begin{theorem}\label{thm:lower_bound}
% For any given $\alpha, K, T$ and (possibly randomized) policy $\pi$, there exists an instance $\phi$ of problem \eqref{eq:regret} with $K+1$ arms such that $    \qreg_\pi(T, \alpha, \phi) + \creg_\pi(T,\alpha,\phi) $ is  $ \Omega\left((1-\alpha)^2 K^{\frac{1}{3}}T^{\frac{2}{3}}\right)$ when $0 \leq \alpha \leq 1$ and  $1 \leq K \leq T$.
% \end{theorem}
\begin{proof} [Proof of Theorem \ref{thm:lower_bound}]
\noindent \textbf{Notation}: 
For any instance $\phi$, we define the arms $m_*^{\phi}$ and $i_*^{\phi}$ as 
$m_*^{\phi} = \arg \max_{i} \mu^i_{\phi}$  and $ i_*^{\phi} = \arg \min_{i} c^i_{\phi} \text{ s.t. } q_{i_{\phi}} \geq (1-\theta)q_{m^*_{\phi}}. $
When the instance is clear, we will use the simplified notation $i_*$ and $m_*$ instead of $i_*^{\phi}$ and $m_*^{\phi}$.

\paragraph{Proof Sketch:}
Lemma \ref{lem:alpha_0_neat} establishes that when $\alpha=0$, for any given policy, there exists an instance on which the sum of quality and cost regret are $ \Omega( K^{1/3}T^{2/3})$.
Now, we generalize the above result for $\alpha=0$ to any $\alpha$ for $0 \leq \alpha \leq 1$. 
The main idea in our reduction is to show that if there exists an algorithm $\pi_\alpha$ for $\alpha > 0$ that achieves $o( K^{1/3}T^{2/3})$ regret on every instance in the family $\cIalpha$, then we can use $\pi_\alpha$ as a subroutine to construct an algorithm $\pi_0$ for problem \eqref{eq:regret} that achieves $o(K^{1/3}T^{2/3})$ regret on every instance in the family $\cIzero$, thus contradicting the lower bound of Lemma \ref{lem:alpha_0_neat}. This will prove
the theorem by contradiction. In order to construct the aforementioned sub-routine, we leverage techniques from {\it Bernoulli factory} to generate a sample from a Bernoulli random variable with parameter $\mu/(1-\alpha)$ using samples from a Bernoulli random variable with parameter $\mu$, for any $\mu, \alpha <1.$ 

\paragraph{Aside on Bernoulli Factory:}
The key tool we use in constructing the algorithm $\algh$ from $\alg$ is \textit{Bernoulli factory} for the linear function. The Bernoulli factory for a specified scaling factor $ C >1$ i.e.  $\bfac (C)$ uses a sequence of independent and identically distributed samples from $\ber(r)$ and returns a sample from $\ber(Cr)$.%$\bfac (C)$ returns two things - a random sample from the distribution $\ber(Cr)$ and the number of samples of $\ber(r)$ needed to generate this random sample. 
The key aspect of a Bernoulli factory is the number of samples needed from $\ber(r)$ to generate a sample from $\ber(Cr)$. We use the Bernoulli factory described in \cite{huber2013nearly} which has a guarantee on the expected number of samples $\tau$ from $\ber(r)$ needed to generate a sample from  $\ber(Cr)$. In particular, for a specified $\delta >0$, 
\begin{align}
\label{eqn:bern_fact}
    \sup_{r \in [0, \frac{1-\delta}{C}]} E[\tau] \leq \frac{9.5 C}{\delta}.
\end{align}

\paragraph{Detailed proof:}
For some value of $p, \epsilon$ (to be specified later in the proof) such that $0 \leq p <1$ and $0 \leq \epsilon \leq p/2$, consider the family of instances $\cIalpha$ and $\cIzero$. Let $\alg$ be any algorithm for the family $\cIalpha$. Using $\alg$, we construct an algorithm $\algh$ for the family $\cIzero$. This algorithm is described in Algorithm \ref{alg:mod_alpha}.
%We denote number of pulls and number of successes for the arms before time $l$ as $N_l = ( N^0_{l}, N^1_{l} \cdots N^K_l )$ and $S_l = ( S^0_{l}, S^1_{l} \cdots S^K_l )$. 
We will use $I_l^{\alpha} = \alg([ (I_1^{\alpha}, r_1), (I_2^{\alpha}, r_2), \cdots (I_{l-1}^{\alpha}, r_{l-1} ) ]  )$ to denote the arm pulled by algorithm $\alg$ at time $l$ after having observed rewards $r_l \ \forall  1 \leq i < l $ through arm pulls $I_l  \ \forall  1 \leq i < l $. The function $\bfac (C)$ returns two values - a random sample from the distribution $\ber(Cr)$ and the number of samples of $\ber(r)$ needed to generate this random sample.

\begin{algorithm}[H]
\label{alg:mod_alpha}
\SetAlgoLined
\KwResult{Arm $I^0_t$ to be pulled in each round $t$, total number of arm pulls $T$ }
\SetKwInOut{Input}{input}\SetKwInOut{Output}{output}
\Input{Algorithm $\alg$, $L$ - Number of arm pulls for algorithm $\alg$}
$l=1, t=1$ \; 
%$ N_1 = S_1 = (0, \cdots 0) $ \;

\For{$l \in [L]$}{
    $I^{\alpha}_l =  \alg([ (I^{\alpha}_1, r_1), (I^{\alpha}_2, r_2), \cdots (I^{\alpha}_{l-1}, r_{l-1} ) ]  )  $ \;
\eIf{$ I_l^{\alpha} = 0 $}{
  Pull arm 0 to obtain outcome $r_l$ \;
    $I^0_t = I_l^{\alpha} =0$ \;
     $U_l = \{t\}$ \;
  }{ Call  $r_l, n =  \bfac(\frac{1}{1-\alpha})$ on samples generated from repeated pulls of the arm $I_l^{\alpha}$ \;
   $U_l = \{ t, t+1 \cdots t+ n -1 \}$ \;
   $I^0_{t} =I^0_{t+1} = \cdots I^0_{t+n-1} = I_l^{\alpha}$ \;
  }
  $S_l = |U_l|$ \;
  $l = l+1$ \;
  $t = t + S_l$ \;
}
$T=t$
\caption{Derived Algorithm $\algh$}
\label{alg:ts_gauss}
\end{algorithm}

\item 
Now, let us analyze the expected modified regret incurred by algorithm $\algh$ on an instance $\cIzero^a$ for any $0 \leq a \leq K$ where the expectation is with respect to the random variable $T$, total number of arm pulls.

Similarly, we analyze the cost regret incurred by algorithm $\algh$ on an instance $\cIzero^a$ for any $0 \leq a \leq K$. 

\begin{align*}
     & \ex \left[ \mqreg_{\algh}(T, 0,\cIzero^a) \right] + \ex \left[ \creg_{\algh}(T, 0,\cIzero^a) \right]\\ 
     & = \ex \left[ \sum_{t=1}^T \left( \mu_{m^*}^{\cIzero^a} - \mu_{I^0_t}^{\cIzero^a} \right) \ind\{ I^0_t = 0 \} \right] + \ex \left[ \sum_{t=1}^T c^{\cIzero^a}_{i_*} -  c^{\cIzero^a}_{I^0_t} \right] \\
    & =  \ex \left[ \sum_{l=1}^L \sum_{t \in U_l} \left( \mu_{m^*}^{\cIzero^a} - \mu_{I^0_t}^{\cIzero^a} \right) \ind\{ I^0_t = 0 \} \right] + \ex \left[  \sum_{l=1}^L \sum_{t \in U_l} c^{\cIzero^a}_{i_*} -  c^{\cIzero^a}_{I^0_t} \right] \\
    &= \ex \left[ \sum_{l=1}^L S_l \left( \mu_{m^*}^{\cIzero^a} - \mu_{I_l^{\alpha}}^{\cIzero^a} \right) \ind\{ I_l^{\alpha} = 0 \} \right] + \ex \left[\sum_{l=1}^L  S_l \left( c^{\cIzero^a}_{i_*} -  c^{\cIzero^a}_{I^{\alpha}_l} \right) \right] \\
    &= \sum_{l=1}^L  \ex \left[  \ex \left[ S_l | I_l^{\alpha} \right] \left( \mu_{m^*}^{\cIzero^a} - \mu_{I_l^{\alpha}}^{\cIzero^a} \right) \ind\{ I_l^{\alpha} = 0 \} \right] + \ex \left[\sum_{l=1}^L  \ex[ S_l| I^{\alpha}_l] \left( c^{\cIzero^a}_{i_*} -  c^{\cIzero^a}_{I^{\alpha}_l} \right) \right] \\
    &\leq\sum_{l=1}^L  \ex \left[  \frac{9.5}{\delta (1-\alpha)} \left( \mu_{m^*}^{\cIzero^a} - \mu_{I_l^{\alpha}}^{\cIzero^a} \right) \ind\{ I_l^{\alpha} = 0 \} \right] +  \ex \left[\sum_{l=1}^L  \frac{9.5}{\delta (1-\alpha)} \left(  c^{\cIzero^a}_{i_*} -  c^{\cIzero^a}_{I^{\alpha}_l} \right) \right] \hspace{10mm} \text{(Using \eqref{eqn:bern_fact})} \\
    & = \frac{9.5}{\delta (1-\alpha)} \sum_{l=1}^L  \ex \left[   \left( (1-\alpha) \mu_{m^*}^{\cIalpha^a} - \mu_{I_l^{\alpha}}^{\cIzero^a} \right) \ind\{ I_l^{\alpha} = 0 \} \right]  + \frac{9.5}{\delta (1-\alpha)} \sum_{l=1}^L \ex \left[ c^{\cIalpha^a}_{i_*} -  c^{\cIalpha^a}_{I^{\alpha}_l} \right] \\
    & \hspace{5mm} \text{(Because costs of arms are same in all instances,  $i_*^{\cIalpha^a} = i_*^{\cIzero^a} = a  $ and $ \mu_{m^*}^{\cIzero^a} = (1-\alpha) \mu_{m^*}^{\cIalpha^a}  $ )} \\
    & = \frac{9.5}{\delta (1-\alpha)} \qreg_{\alg}(L, \alpha,\cIalpha^a) + \frac{9.5}{\delta (1-\alpha)} \creg_{\alg}(L, \alpha,\cIalpha^a).
\end{align*}

% \begin{align*}
%      \ex \left[ \creg_{\algh}(T, 0,\cIzero^a) \right]&  
%      = \ex \left[ \sum_{t=1}^T c^{\cIzero^a}_{i_*} -  c^{\cIzero^a}_{I^0_t} \right] \\
%     & = \ex \left[  \sum_{l=1}^L \sum_{t \in U_l} c^{\cIzero^a}_{i_*} -  c^{\cIzero^a}_{I^0_t} \right] \\
%     &= \ex \left[\sum_{l=1}^L  S_l \left( c^{\cIzero^a}_{i_*} -  c^{\cIzero^a}_{I^{\alpha}_l} \right) \right] \\
%     &= \ex \left[\sum_{l=1}^L  \ex[ S_l| I^{\alpha}_l] \left( c^{\cIzero^a}_{i_*} -  c^{\cIzero^a}_{I^{\alpha}_l} \right) \right] \\
%     &\leq \ex \left[\sum_{l=1}^L  \frac{9.5}{\delta (1-\alpha)} \left(  c^{\cIzero^a}_{i_*} -  c^{\cIzero^a}_{I^{\alpha}_l} \right) \right] \hspace{10mm} \text{(Using \eqref{eqn:bern_fact})} \\
%     & = \frac{9.5}{\delta (1-\alpha)} \sum_{l=1}^L \ex \left[ c^{\cIalpha^a}_{i_*} -  c^{\cIalpha^a}_{I^{\alpha}_l} \right] \\
%     & \hspace{5mm} \text{(Because costs of arms are same in all instances and $i_*^{\cIalpha^a} = i_*^{\cIzero^a} = a  $)} \\
%     & = \frac{9.5}{\delta (1-\alpha)} \creg_{\alg}(L, \alpha,\cIalpha^a).
% \end{align*}

 Thus, 
\begin{align}
\label{eqn:c_ineq} \nonumber
    &\qreg_{\alg}(L, \alpha,\cIalpha^a) + \creg_{\alg}(L, \alpha,\cIalpha^a) \\
    & \geq  \frac{\delta (1-\alpha)}{9.5}  \ex \left[  \mqreg_{\algh}(T, 0,\cIzero^a) +\creg_{\algh}(T, 0,\cIzero^a) \right] \\ \nonumber
    & \geq  \frac{\delta (1-\alpha)}{9.5}  \ex \left[ \mqreg_{\algh}(L, 0,\cIzero^a) + \creg_{\algh}(L, 0,\cIzero^a) \right]  \hspace{10mm} (\text{because } L \leq T) \\ 
   & \geq  \frac{\delta (1-\alpha)}{9.5}   \left( \mqreg_{\algh}(L, 0,\cIzero^a) + \creg_{\algh}(L, 0,\cIzero^a) \right)
\end{align}

Using Lemma \ref{lem:alpha_0_neat} and choosing $p=\frac{1-\alpha}{3}, \delta = \frac{1}{2}, \epsilon = \frac{p}{2} (\frac{K}{T})^{1/3}$, we get for any randomized algorithm $\alg$, there exists instance $\cIalpha^b$ (for some $0 \leq b \leq K$) such that $ \qreg_\pi(T, \alpha, \cIalpha^b) + \creg_\pi(T,\alpha,\cIalpha^b)$ is $\Omega\left( (1-\alpha)^2 K^{1/3} T^{2/3} \right)$.

\end{proof}

\section{Performance of Algorithms}
%\label{sec:upper_bound_appendix} 
\label{sec:alg_performance_appendix}

We use the following fact in the proof of Theorem \ref{thm::ts-lin-reg}.

\begin{fact}
\label{fact1}
\citep{abramowitz1948handbook}
For a Normal random variable $Z$ with mean $m$ and variance $\sigma^2$, for any $z$, 
$$  \prob \left( |Z- m| > z \sigma \right) > \frac{1}{4 \sqrt{\pi}} \exp(-\frac{7z^2}{2}).$$
\end{fact}

\begin{proof} [Proof of Theorem \ref{thm::ts-lin-reg}]

This proof is inspired by the lower bound proof in \cite{agrawal2017near}.
For any given $\alpha, K$ and $T$, we construct an instance on which the \consTS \ algorithm (Algorithm \ref{alg::ts}) gives linear regret in cost.

%\noindent Consider an instance $\phi$ with $K$ arms where the cost of first arm is 0 and all other arms is 1. Also assume that the quality of arm 1 is $\mu_1 = \frac{2d}{\sqrt{T}}$ and the quality of all other arms are $\forall i\ne 1: q_i = \mu =\frac{d}{(1-\alpha)\sqrt{T}}$ for some $d>0$. With this, we have all qualities are positive and also $\mu_1 = (1-\alpha)\mu + \frac{d}{\sqrt{T}}$. We also consider the case where $\alpha<0.5$ and $d<\alpha \mu\sqrt{T}$ so that $0<\mu<1$ and all arms are feasible.  

Consider an instance $\phi$ with $K$ arms where the costs and mean reward of the $j$-th arm are\\
$ c_j =
\begin{cases}
0 & j=0 \\
1 & j\neq 0 
\end{cases}, \hspace{10mm}
\ \mu_j =
\begin{cases}
% \frac{2d}{\sqrt{T}} & j=0 \\
% \frac{d}{(1-\alpha)\sqrt{T}} & j\neq 0 \\
(1-\alpha)q + \frac{d}{\sqrt{T}} & j=0 \\
q & j\neq 0 \\
\end{cases} $ \\
where $q =\frac{d}{(1-\alpha)\sqrt{T}}$  for some $0 < d < \min \{ \sqrt{T}/2, (1-\alpha)\sqrt{T} \}$.
Moreover, the reward of each arm is deterministic though this fact is not known to the agent. As in the SMS application, we assume that the cost rewards of all arms are known a priori to the agent.

%Also, define $q =\frac{d}{(1-\alpha)\sqrt{T}}$. Then, the mean reward of the first arm is $(1-\alpha)q + \frac{d}{\sqrt{T}}$ and all the other arms is $q$. 

Let the prior distribution that the agent assumes over the mean reward of each arm be $\mathcal{N}(0,\sigma_0^2)$ for some prior variance $\sigma_0^2$ . Further, the agent assumes that the observed qualities to be normally distributed with noise variance $\sigma_n^2$. As such at the start of period $t$, the agent will consider a normal posterior distribution for each arm $i$ with mean 
\begin{equation}
    \label{m_i_t}
    \hmu_i(t) = \frac{T_i(t)}{\frac{\sigma_n^2}{\sigma_0^2}+T_i(t)}\mu_i
\end{equation}
and variance 
\begin{equation}
    \label{post_variance}
    \sigma_i(t)^2 =  \left(\frac{1}{\sigma_0^2} + \frac{T_i(t)}{\sigma_n^2}\right)^{-1}.
\end{equation}

As  $d < q \alpha \sqrt{T}$, the highest quality across all arms is $q$. Thus, note that all arms are \textit{feasible} in terms of quality i.e. have their quality within $(1-\alpha)$ factor of the best quality arm. Hence, quality regret $\qreg_{{\sf \consTS }}(t, \alpha,\phi)  = 0 \ \forall t >0 $ (for any algorithm) on this instance. 

The first arm is the optimal arm ($i_*$). Thus, the cost regret equals the number of times any arm but the first arm is pulled. In particular, let 
$$R_c(T) = \sum_{t=1}^T \max\{c_{I_{t}} - c_{\optarm},0\} = \sum_{t=1}^T \mathbf{1}\{ I_t \ne 1 \},$$
so that $\creg_{{\sf \consTS }}(T, \alpha,I) = \ex \left[ R_C(T) \right].$

Define the event $A_{t-1} = \{ \sum_{i\neq 1} T_i(t) \leq s T \sqrt{K} \}$ for a fixed constant $s >0$. For any $t$, if the event $A_{t-1}$ is not true, then $R_c(T) \geq R_c(t) \geq s T \sqrt{K}$. We can assume that $\prob(A_{t-1}) \geq 0.5 \ \forall t \leq T.$ Otherwise 
\begin{align*}
    \creg_{{\sf \consTS }}(T, \alpha,\phi) & = \ex \left[ R_C(T) \right] \\
    & \geq 0.5 E[R_C(T) | A_{t-1}^c] \\
    & = \Omega(T \sqrt{K} ).
\end{align*}

Now, we will show that whenever $A_{t-1}$ is true, probability of playing a sub-optimal arm is at least a constant. For this, we show that the probability that  $\mu^{score}_1(t) \leq \mu_1$ and $\mu^{score}_i(t) \geq \frac{\mu_1}{1-\alpha}, \ \text{for some } 1 < i \leq K$ is lower bounded by a constant.

Now, given any history of arm pulls $\mathcal{F}_{t-1}$ before time $t$, $\mu^{score}_1(t)$ is a Gaussian random variable with mean $ \hmu_1(t) = \frac{T_i(t)}{\frac{\sigma_n^2}{\sigma_0^2}+T_i(t)}\mu_1$. By symmetry of Gaussian random variables, we have
\begin{equation*}
\begin{aligned}
\prob \left( \mu^{score}_1(t) \leq \mu_1  \bigg | \mathcal{F}_{t-1}\right) &\geq  \prob \left( \mu^{score}_1(t) \leq \frac{T_i(t)}{\frac{\sigma_n^2}{\sigma_0^2}+T_i(t)} \mu_1  \bigg | \mathcal{F}_{t-1}\right) \\
&=\prob \left( \mu^{score}_1(t) \leq \hmu_1(t)  \bigg | \mathcal{F}_{t-1}\right) \\
&= 0.5.
\end{aligned}
\end{equation*} 

Based on \eqref{m_i_t} and \eqref{post_variance}, given any realization $F_{t-1}$ of $\mathcal{F}_{t-1}$, $\mu^{score}_i(t)$ for $i \neq 1$ are independent Gaussian random variables with mean $ \hmu_i(t) $ and variance $\sigma_i(t)^2$. Thus, we have% $\frac{T_i(t)}{\frac{\sigma_n^2}{\sigma_0^2}+T_i(t)}\mu$ and variance $\sigma_i(t)^2 =  \left(\frac{1}{\sigma_0^2} + \frac{T_i(t)}{\sigma_n^2}\right)^{-1}$. 

\begin{align*}
     & \prob \left(  \exists i \neq 1, \ \mu^{score}_i(t) \geq \frac{\mu_1}{1- \alpha} \ \middle| \mathcal{F}_{t-1} = F_{t-1} \right) \\
     & = \prob \left(  \exists i \neq 1, \ \mu^{score}_i(t) - \hmu_i(t)  \geq \frac{1}{1- \alpha} \left(  q(1-\alpha) + \frac{d}{\sqrt{T}}\right) - \hmu_i(t)   \ \middle| \mathcal{F}_{t-1} = F_{t-1} \right) \\
     & = \prob \left(  \exists i \neq 1, \  \mu^{score}_i(t) - \hmu_i(t)  \geq  \frac{d}{(1-\alpha) \sqrt{T}}+\frac{1}{1+T_i(t)\frac{\sigma_0^2}{\sigma_n^2}}q  \ \middle| \mathcal{F}_{t-1} = F_{t-1} \right) \\ 
     & \geq \prob \left(  \exists i \neq 1, \  \mu^{score}_i(t) - \hmu_i(t)  \geq  \frac{d}{(1-\alpha) \sqrt{T}}+q\ \middle| \mathcal{F}_{t-1} = F_{t-1} \right) \\
     & = \prob \left(  \exists i \neq 1, \  \mu^{score}_i(t) - \hmu_i(t)  \geq  \frac{2d}{(1-\alpha) \sqrt{T}}\ \middle| \mathcal{F}_{t-1} = F_{t-1} \right) \\
     &= \prob \left(  \exists i \neq 1, \  \left( \mu^{score}_i(t) - \hmu_i(t)  \right) \frac{1}{\sigma_i(t)} \geq \left(  \frac{2d}{(1-\alpha) \sqrt{T}}  \right)\frac{1}{\sigma_i(t)} \ \middle| \mathcal{F}_{t-1} = F_{t-1} \right) \\
     &= \prob \left(  \exists i \neq 1, \ Z_i(t) \geq \left(  \frac{2d}{(1-\alpha) \sqrt{T}} \right)\frac{1}{\sigma_i(t)} \ \middle| \mathcal{F}_{t-1} = F_{t-1} \right) \\
    \end{align*}
     where $Z_i(t)$ are independent standard normal variables for all $i,t$. Thus,
     
     \begin{align*}
     & \prob \left(  \exists i \neq 1, \ \mu^{score}_i(t) \geq \frac{\mu_1}{1- \alpha} \ \middle| \mathcal{F}_{t-1} = F_{t-1} \right) \\
     & = 1 - \prob \left(  \forall i \neq 1, \ Z_i(t) \leq \left(  \frac{2d}{(1-\alpha) \sqrt{T}} \right)\frac{1}{\sigma_i(t)} \ \middle| \mathcal{F}_{t-1} = F_{t-1} \right) \\
     & = 1 - \Pi_{i \neq 1} \left(  1 - \prob \left(    \ Z_i(t) \geq \left(  \frac{2d}{(1-\alpha) \sqrt{T}} \right)\frac{1}{\sigma_i(t)} \ \middle| \mathcal{F}_{t-1} = F_{t-1} \right) \right) \\
     & \geq 1 - \Pi_{i \neq 1} \left(  1 -  \frac{1}{8 \sqrt{\pi}} \exp \left( -\frac{7}{2}  \frac{1}{\sigma_i(t)^2} \left(  \frac{2d}{(1-\alpha) \sqrt{T}} \right)^2  \right)  \right) \hspace{30mm} (\text{Using Fact \ref{fact1}}) \\
     & = 1 - \Pi_{i \neq 1} \left(  1 -  \frac{1}{8 \sqrt{\pi}} \exp \left( -\frac{7}{2}  \left(\frac{1}{\sigma_0^2} + \frac{T_i(t)}{\sigma_n^2}\right) \left(  \frac{2d}{(1-\alpha) \sqrt{T}}\right)^2  \right)  \right) \\
     & \geq 1 - \Pi_{i \neq 1} \left(  1 -  \frac{1}{8 \sqrt{\pi}} \exp \left( -\frac{7}{2}  \left(\frac{1}{\sigma_0^2} + \frac{1}{\sigma_n^2}\right) T_i(t) \left(  \frac{2d}{(1-\alpha) \sqrt{T}} \right)^2  \right)  \right),
\end{align*}
 The last inequality follows from the fact that $T_i(t)\geq 1$. % and we use Fact \ref{fact1} result to get the third last inequality.

Now, when the event $A_{t-1}$ holds, we have $\sum_{i \neq 1} T_i(t) \leq s T \sqrt{K}$. Thus, the right hand side would be minimized when $T_i(t) = \frac{sT}{\sqrt{K}}, \ \forall i \neq 1$. 
Substituting this value of $T_i(t)$, the right hand side reduces to $ g(K) = 1 - \Pi_{i \neq 1} \left(  1 -  \frac{1}{8 \sqrt{\pi}} \exp \left( -14 \left(\frac{1}{\sigma_0^2} + \frac{1}{\sigma_n^2}\right)    \frac{s \sqrt{K} d^2}{(1-\alpha)^2}   \right)  \right)$. Thus, $ \prob \left(  \exists i \neq 1, \ \mu^{score}_i(t) \geq \frac{\mu_1}{1- \alpha} \ \middle| \mathcal{F}_{t-1} = F_{t-1} \right) \geq g(K)$ whenever $F_{t-1}$ is such that $A_{t-1}$ holds.

Probability of playing any sub-optimal arm at time $t$ is, 
\begin{align*}
 \prob \left(  \exists i \neq 1, \ I_t =i \right) & \geq \prob \left( \mu^{score}_1(t) \leq \mu_1, \ \exists i \neq 1 \text{ s.t. } \mu^{score}_i(t) \geq \frac{\mu_1}{1-\alpha} \right)  \\
    & = \ex \left[ \prob \left( \mu^{score}_1(t) \leq \mu_1, \ \exists i \neq 1 \text{ s.t. } \mu^{score}_i(t) \geq \frac{\mu_1}{1-\alpha} \right) \middle | \mathcal{F}_{t-1}  \right] \\
    & \geq \ex \left[ \prob \left( \mu^{score}_1(t) \leq \mu_1, \ \exists i \neq 1 \text{ s.t. } \mu^{score}_i(t) \geq \frac{\mu_1}{1-\alpha} \right) \middle| \mathcal{F}_{t-1}, A_{t-1}  \right] \prob(A_{t-1}) \\
    & \geq   \frac{1}{2} \cdot g(K) \cdot \frac{1}{2} 
\end{align*}

Thus, at every time instant $t$, the probability of playing a sub-optimal is lower bounded by $\frac{g(K)}{4}$. This implies that the cost regret $\creg_{{\sf \consTS }}(T, \alpha,\phi) \geq 0.25 T g(K)$.

\end{proof}

\begin{proof}[Proof of Theorem \ref{thm:explore_first_regret}]
This algorithm has two phases -  pure exploration and UCB. In the first phase, the algorithm pulls each arm a specified number of times ($\tau$). In the second phase, the algorithms maintains upper and lower confidence bounds on the mean reward of each arm. Then, it estimates a \textit{feasible} set of arms and pulls the cheapest arm in this set.

We will define the \textit{clean event }$\mathcal{E}$ in this proof as the event that for every time $t \in [T]$ and arm $i \in [K]$, the difference between the mean reward and the empirical mean reward does not exceed the size of the confidence interval ($\beta_i(t)$) i.e. 
 $\mathcal{E} = \{ |\hmu_i(t) - \mu_i| \leq \beta_i(t), \ \forall i \in [K], \ t \in [T] \}$. 
 
Define $\hat{t} = K \tau + 1$ as the first round in the UCB phase of the algorithm. Further, define instantaneous cost and quality regret as the regret incurred in the $t$-th arm pull: 

\begin{equation}\label{eq:inst_regret}
\begin{aligned}
& \qreg_\pi^{inst}(t, T, \alpha,  \pmb{\mu}, \mathbf{c})  = \ex \left[\max\{ (1-\alpha) \mu_{m_{*}} - \mu_{\pi_t},0\} \right], \\
& \creg_\pi^{inst}(t, T, \alpha, \pmb{\mu}, \mathbf{c})   = \ex \left[  \max\{c_{\pi_{t}} - c_{\optarm},0\} \right],
\end{aligned}
\end{equation}

where the expectation is over the randomness in the policy $\pi$.

Let us first assume that the clean event holds.
As both the instantaneous regrets are upper bounded by 1, $\sum_{t= 1}^{K \tau} \qreg_\pi^{inst}(t, T, \alpha,  \pmb{\mu}, \mathbf{c}) \leq  K \tau$ and $\sum_{t= 1}^{K \tau} \creg_\pi^{inst}(t, T, \alpha,  \pmb{\mu}, \mathbf{c}) \leq  K \tau$.

%In the exploration phase of the algorithm, the cost and quality regret are at most equal to the number of rounds in the exploration phase. Thus, 
Now, let us look at the UCB phase of the algorithm.
Here, $\forall \ \hat{t} \leq t \leq T$, we have
% \begin{flalign*}
%     && \muucb_{\optarm}(t) & \geq \mu_{\optarm} &&  \text{(because clean event holds)} \\
%     && &\geq (1-\alpha)\mu_{m_*}&&   \text{(by definition of } \optarm) \\
%     && &\geq (1-\alpha)\mu_{m_t} &&  \text{(by definition of } m_*) \\
%     && &\geq (1-\alpha) \mulcb_{m_t}(t) && \text{(because clean event holds)}.
% \end{flalign*}
$$\muucb_{\optarm}(t)  \geq \mu_{\optarm}\geq (1-\alpha)\mu_{m_*} \geq (1-\alpha)\mu_{m_t} \geq (1-\alpha) \mulcb_{m_t}(t).$$
Here, the first and fourth inequality are because of the clean event. The second and third inequality are from the definition of $\optarm$ and $m_*$ respectively.

Thus from the inequality above, the optimal arm $\optarm$ is in the set $Feas(t), \forall \hat{t} \leq t \leq T$. This implies that the arm pulled in each time step in the UCB phase, is either the optimal arm or an arm cheaper than it. Thus, instantaneous cost regret is zero for all time steps in the UCB phase of the algorithm.

Now, let us look at the quality regret in the UCB phase i.e. for any $\hat{t} \leq t \leq T$. 
We have $$\mu_{I_t} + 2 \beta_{I_t}(t) \geq \mu^{\sf UCB}_{I_t}(t) \geq (1-\alpha) \mu^{\sf LCB}_{m_t}(t) \geq (1-\alpha) \mu^{\sf LCB}_{m_*}(t) \geq (1-\alpha) \left( \mu_{m_*} - 2\beta_{m_*}(t) \right) \geq (1-\alpha) \mu_{m_*} - 2\beta_{m_*}(t) $$
The first and fourth inequality hold because the clean event holds. The second and third inequalities follow from the definition of $I_t$ and $m_t$ respectively. Thus, 

$\qreg_\pi^{inst}(t, T, \alpha, \pmb{\mu}, \mathbf{c} | \mathcal{E}) = (1-\alpha)\mu_{m_*} - \mu_{I_t} \leq 2 \left( \beta_{I_t}(t) + \beta_{m_*}(t)  \right) 
     \leq 2 \left( \sqrt{ \frac{ 2 \log{T}}{\tau}} + \sqrt{ \frac{ 2 \log{T}}{\tau} } \right) = 4 \sqrt{ \frac{ 2 \log{T}}{\tau}} $.

The total regret incurred by the algorithm is the sum of the instantaneous regrets across all time steps in the exploration and the UCB phase. Thus, 

$\qreg_\pi(T, \alpha,  \pmb{\mu}, \mathbf{c} | \mathcal{E}) \leq K\tau + 4(T-K\tau) \sqrt{ \frac{ 2 \log{T}}{\tau}} \leq K\tau + 4T\sqrt{ \frac{ 2 \log{T}}{\tau}}  $ and
$ \creg_\pi(T, \alpha,  \pmb{\mu}, \mathbf{c} | \mathcal{E}) \leq K\tau $.
Substituting $\tau = (T/K)^{2/3}$, we conclude that both cost and quality regret are $O(K^{1/3} T^{2/3} \sqrt{\log T})$.

Now, when the clean event does not hold, the cost and quality regret are at most $T$ each. The probability that the clean event does not hold is at most ${2}/{T^2}$ (Lemma 1.6 in \cite{MAL-068}). Thus, the expected cost and quality regret obtained by averaging over the clean event holding and not holding is $O(K^{1/3} T^{2/3} \sqrt{\log T})$.

\end{proof}

\begin{proof}[Proof of Theorem \ref{thm:cost_quality_consistent}]
As in the previous proof, we will define the \textit{clean event }$\mathcal{E}$ as the event that for every time $t \in [T]$ and arm $i \in [K]$, the difference between the mean reward and the empirical mean reward does not exceed the size of the confidence interval ($\beta_i(t)$) i.e. 
 $\mathcal{E} = \{ |\hmu_i(t) - \mu_i| \leq \beta_i(t), \ \forall i \in [K], \ t \in [T] \}$.  
 Also, define the quality and cost gap of each arm as $ \Delta_{\mu,i} = \max \{ (1-\alpha) \mu_{m^*}  - \mu_i, 0 \}$ and $\Delta_{c,i} = \max \{ c_{\optarm}  - c_i, 0 \}$.
 
When the clean event does not hold, both cost and quality regrets are upper bounded by $T$.
Let us look at the case when the clean event holds and analyze the cost and quality regret.
 
\textbf{Quality Regret:} Let $t_i$ be the last time $t$ when $i \in Feas(t)$ i.e. $t_i = \max\{ K, \ \max \{ t: i \in Feas(t) \} \}.$ Thus, $T_i(T) = T_i(t_i)$. 

%Since $i \in Feas(t_i)$, we know $\muucb_i(t_i) \geq (1-\alpha) \muucb{m_{t_i}}(t_i)$.

Consider any arm $i$ which would incur a quality regret on being pulled i.e. arm $i$ such that $\mu_i < (1-\alpha) \mu_{m_*}$.  We have $$ \mu_i + 2 \beta_{i}(t_i) \geq \muucb_i(t_i) \geq (1-\alpha) \muucb_{m_{t_i}}(t_i) \geq (1-\alpha) \muucb_{m_{*}}(t_i) \geq (1-\alpha) \mu_{m_{*}}.$$
The first and fourth inequality hold because of the clean event. The third inequality is from the definition of $m_{t_i}$.

Thus, $(1-\alpha) \mu_{m_{*}} - \mu_i \leq 2 \beta_i(t_i) $. Using the definition of $\beta_i(t_i)$, we get 
$T_i(T)= T_i(t_i) \leq \frac{8 \log{T}}{ \Delta_{\mu,i}^2 } $.

Using Jensen's inequality, 
\begin{align*}
    \left( \frac{\sum_{i=1}^K T_i(T) \Delta_{\mu,i}}{T}  \right)^2 & \leq \frac{\sum_{i=1}^K T_i(T) \Delta^2_{\mu,i}}{T} \\
    & = \frac{\sum_{i=1: \ \Delta_{\mu, i} >0}^K T_i(T) \Delta^2_{\mu,i}}{T} \\
    & \leq \sum_{i=1: \Delta_{\mu, i} >0}^K \frac{8 \log{T}}{\Delta^2_{\mu,i}} \frac{ \Delta^2_{\mu,i}}{T} \\
    &= \frac{8K \log{T}}{T} 
\end{align*}
Thus, $ \qreg_\pi(T, \alpha,  \pmb{\mu}, \mathbf{c} | \mathcal{E})   \leq \sqrt{8KT \log{T}}$.

\textbf{Cost Regret:} Let $i$ be an arm such that $c_i > c_{\optarm}$. Let $\tdt_i$ be the last time when arm $i$ is pulled. Thus, $\optarm \notin Feas(\tdt_i)$.
%$\optarm \notin Feas(t)$ and arm $i$ is pulled. Note that, after time $\tdt_i$, $\optarm \in Feas(t)$ everytime arm $i$ is pulled and no cost regret will be incurred due to pulling arm $i$. 
We have, 
$ \mu_{\optarm} \leq \muucb_{\optarm}(\tdt_i) < \muucb_i(\tdt_i) \leq \mu_i(\tdt_i) + 2\sqrt{(2\log{T})/T_i(\tdt_i)} $. Thus, $$ T_i(T)= T_i(\tdt_i) <  \frac{8 \log{T}}{ (\mu_{\optarm} - \mu_i)^2} \leq \frac{8 \delta^2 \log{T}}{ (c_{\optarm} - c_i)^2} = \frac{8 \delta^2 \log{T}}{\Delta_{c,i}^2}.  $$
Using Jensen's inequality as for the case of quality regret, we get,  $ \creg_\pi(T, \alpha,  \pmb{\mu}, \mathbf{c} | \mathcal{E})   \leq \sqrt{8\delta^2 KT \log{T}}$.

Note that the probability of the clean event is at least $1 - 2/T^2$ (Lemma 1.6 in \cite{MAL-068}). Thus, the sum of the expected cost and quality regret by averaging over the clean event holding and not holding is $O(  ( 1 +  \delta )\sqrt{KT\log{T}})$.

\end{proof}

\newpage
%\section{Algorithm with Unknown and Random Costs}
%\label{sec:unknown_costs_algo_appendix}

% \begin{algorithm}
% \SetAlgoLined
% \For{$t \in [Kf(K, T)+1, T]$}{
%     $ \hmu_i(t) \leftarrow \left(  \sum_{\tau=1}^{t-1} r_{\tau} \ind \{ I_{\tau} =i \}\right)/ { T_i(t) } \ \forall i \in [K] $\;
%     $ \hc_i(t) \leftarrow \left( \sum_{\tau=1}^{t-1} \chi_{\tau} \ind \{ I_{\tau} =i \} \right) / { T_i(t) } \ \forall i \in [K] $\;
%     $\beta_i(t) \leftarrow \sqrt{\left( 2 \log{T} \right)/{T_i(t)}} \ \forall i \in [K]$\;
%     $\mu^{\sf UCB}_i(t) \leftarrow \min \{ \hmu_i(t) + \beta_i(t), 1 \} \ \forall i \in [K] $\;
%     $\mu^{\sf LCB}_i(t) \leftarrow \max \{ \hmu_i(t)  - \beta_i(t), 0\} \ \forall i \in [K] $\;
%     %$c^{\sf UCB}_i(t) \leftarrow \min \{ \hc_i(t) + \beta_i(t), 1 \} \ \forall i \in [K] $\;
%     $c^{\sf LCB}_i(t) \leftarrow \max \{ \hc_i(t)  - \beta_i(t), 0\} \ \forall i \in [K] $\;
%     $m_t = \arg \max_{i} \mu^{\sf LCB}_i(t)$\;
%     $Feas(t) = \{i: \mu^{\sf UCB}_i(t)> (1-\alpha) \mu^{\sf LCB}_{m_t}(t) \}$\;
%     $I_t = \arg \min_{i \in Feas(t)} c^{\sf LCB}_i $\;
%     Pull arm $I_t$ to obtain reward $r_t$ and cost $ \chi_t$\;
%     $T_{i}(t+1) = T_i(t) + \mathbf{1}\{ I_t =i \} \ \forall i \in [K]$\;
% }
% \caption{UCB phase of \mainALG \ with unknown costs}
% \label{alg:explore_first_unknown_cost}
% \end{algorithm}

\end{document}